\documentclass{article}
\PassOptionsToPackage{numbers, compress}{natbib}
\usepackage[preprint]{neurips_2024}
\usepackage{minitoc} 
\usepackage[utf8]{inputenc} 
\usepackage[T1]{fontenc}    
\usepackage{hyperref}       
\usepackage{url}            
\usepackage{booktabs}       
\usepackage{amsfonts}       
\usepackage{nicefrac}       
\usepackage{microtype}      
\usepackage{xcolor}         

\usepackage{lipsum}
\usepackage{natbib}
\usepackage[most]{tcolorbox}
\usepackage{xcolor}
\usepackage{array}
\definecolor{RoyalBlue}{RGB}{65,105,225}
\definecolor{MyRed}{RGB}{225,0,0}
\usepackage{microtype}
\usepackage{graphicx}
\usepackage{booktabs} 
\usepackage{subcaption}

\usepackage{hyperref}
\usepackage{amsmath}
\usepackage{enumitem}

\usepackage{amsmath}
\usepackage{amssymb}
\usepackage{mathtools}
\usepackage{amsthm}

\usepackage[capitalize,noabbrev]{cleveref}
\RequirePackage{boldline} 
\RequirePackage{tabularx} 
\RequirePackage{tabulary} 
\usepackage{threeparttable}
\usepackage{multirow}
\usepackage{wrapfig}
\usepackage{cancel}

\newcolumntype{C}{>{\centering\arraybackslash}X}
\newcolumntype{L}{>{\raggedright\arraybackslash}X}
\newcolumntype{R}{>{\raggedleft\arraybackslash}X}
\newcolumntype{J}{>{\justifying\arraybackslash}X}
\newcolumntype{P}[1]{>{\centering\arraybackslash}p{#1}}
\newcolumntype{M}[1]{>{\centering\arraybackslash}m{#1}}
\newcolumntype{B}[1]{>{\centering\arraybackslash}b{#1}}

\theoremstyle{plain}
\newtheorem{theorem}{Theorem}[section]

\newtheorem{lemma}[theorem]{Lemma}
\newtheorem{corollary}[theorem]{Corollary}
\theoremstyle{definition}
\newtheorem{definition}[theorem]{Definition}
\newtheorem{ass}[theorem]{Assumption}
\theoremstyle{remark}
\newtheorem{remark}[theorem]{Remark}
\usepackage{algorithm}
\usepackage{algorithmic}

\usepackage{amsmath,amsfonts,bm}

\def\eqref#1{equation (\ref{#1})}

\def\1{\bm{1}}

\def\vlam{{\bm{\lambda}}}

\def\vJ{{\bm{J}}}

\def\vf{{\bm{f}}}

\def\vr{{\bm{r}}}

\def\vu{{\bm{u}}}
\def\vv{{\bm{v}}}

\def\vz{{\bm{z}}}

\def\mA{{\bm{A}}}
\def\mB{{\bm{B}}}

\def\mU{{\bm{U}}}
\def\mV{{\bm{V}}}
\def\mW{{\bm{W}}}

\def\mSigma{{\bm{\Sigma}}}

\DeclareMathAlphabet{\mathsfit}{\encodingdefault}{\sfdefault}{m}{sl}
\SetMathAlphabet{\mathsfit}{bold}{\encodingdefault}{\sfdefault}{bx}{n}

\def\sJ{{\mathbb{J}}}

\newcommand{\E}{\mathbb{E}}

\newcommand{\defeq}{\vcentcolon=}
\newcommand{\tb}{\textbf}
\newcommand{\pf}{ {\mathcal{T}} }
\newcommand{\ccs}{ \mathrm{CCS} }

\newcommand{\bee}{\begin{equation}\begin{aligned}}
\newcommand{\ee}{\end{aligned}\end{equation}}

\usepackage{tikz}

\newcommand{\ie}{\textit{i}.\textit{e}.}
\newcommand{\eg}{\textit{e}.\textit{g}.}

\usepackage[textsize=tiny]{todonotes}

\title{Panacea: Pareto Alignment via Preference Adaptation for LLMs}

\author{%
  Yifan Zhong$^{1, 2}$\thanks{Equal contribution. $^{1}$Institute for Artificial Intelligence, Peking University. $^{2}$National Key Laboratory of General Artificial Intelligence, BIGAI. $^{3}$Department of Computer Science, City University of Hong Kong. $^{4}$Yuanpei College, Peking University. Correspondence to: Yaodong Yang <yaodong.yang@pku.edu.cn>}\ \ ,\  Chengdong Ma$^{1*}$,\ Xiaoyuan Zhang$^{3*}$,\ Ziran Yang$^{4}$, Haojun Chen$^{1}$ \\ \textbf{Qingfu Zhang}$^{3}$\textbf{, Siyuan Qi}$^{2}$\textbf{, Yaodong Yang}$^{1}$ \\
}

\begin{document}

\maketitle

\vspace{-10pt}
\begin{abstract}
\vspace{-4pt}
Current methods for large language model alignment typically use scalar human preference labels. However, this convention tends to oversimplify the multi-dimensional and heterogeneous nature of human preferences, leading to reduced expressivity and even misalignment. This paper presents Panacea, an innovative approach that reframes alignment as a multi-dimensional preference optimization problem. Panacea trains a single model capable of adapting online and Pareto-optimally to diverse sets of preferences without the need for further tuning. A major challenge here is using a low-dimensional preference vector to guide the model's behavior, despite it being governed by an overwhelmingly large number of parameters. To address this, Panacea is designed to use singular value decomposition (SVD)-based low-rank adaptation, which allows the preference vector to be simply injected online as singular values. Theoretically, we prove that Panacea recovers the entire Pareto front with common loss aggregation methods under mild conditions. Moreover, our experiments demonstrate, for the first time, the feasibility of aligning a single LLM to represent an exponentially vast spectrum of human preferences through various optimization methods. Our work marks a step forward in effectively and efficiently aligning models to diverse and intricate human preferences in a controllable and Pareto-optimal manner.

\end{abstract}

\vspace{-8pt}

\section{Introduction}
\vspace{-4pt}
AI alignment aims to ensure AI systems align with human intentions, and there has been notable progress in this area, especially for large language models (LLMs)~\cite{ji2023ai, casper2023open, kaufmann2023survey,achiam2023gpt}. The prevailing approach for LLM alignment involves curating a dataset $\{(x, y_1, y_2, z)\}$, where each prompt $x$ is associated with a pair of responses $(y_1, y_2)$ and a scalar label $z \in \{0,1\}$ that indicates if $y_1$ is a ``better'' response. These labels are typically generated based on detailed guidelines that encompass various criteria, reflecting multiple dimensions $i \in \{1, \cdots, m\}$ of human preferences (\eg, helpfulness, harmlessness, conciseness, humor, formality). Pre-trained models are subsequently further optimized on this dataset using methods including reinforcement learning, supervised learning, or game-theoretical approaches~\cite{jaques2019way,ouyang2022training,lee2023rlaif,bai2022constitutional,rafailov2023direct,azar2023general,swamy2024minimaximalist,munos2023nash}. However, this \emph{single-objective alignment} methodology may not fully capture the complexity of real-world scenarios for two reasons (\Cref{fig:scalar-limitations}).

\textbf{First}, this method can lead to inconsistency and ambiguity in \textbf{data labels}. Human labelers assign scalar labels $z$ by \textit{implicitly} evaluating responses across every dimension $i$ with \emph{different preference weights} to $i$,
and reaching a final judgment. These differences often result in conflicting labels, causing misalignment or learning failures (\Cref{app:single-analysis}), substantiated by the low average label agreement reported in \cite{bai2022training}.
\textbf{Second}, optimizing a single objective leads to only one \tb{model} that attempts to fit the potentially conflicting labeling preferences, \ie, the helpfulness-harmlessness dilemma. This single model may not cover the full spectrum of human preferences across all dimensions, thereby exacerbating biases against underrepresented groups and failing to meet diverse user needs.

\begin{figure*}
    \centering
    \includegraphics[width=1\textwidth]{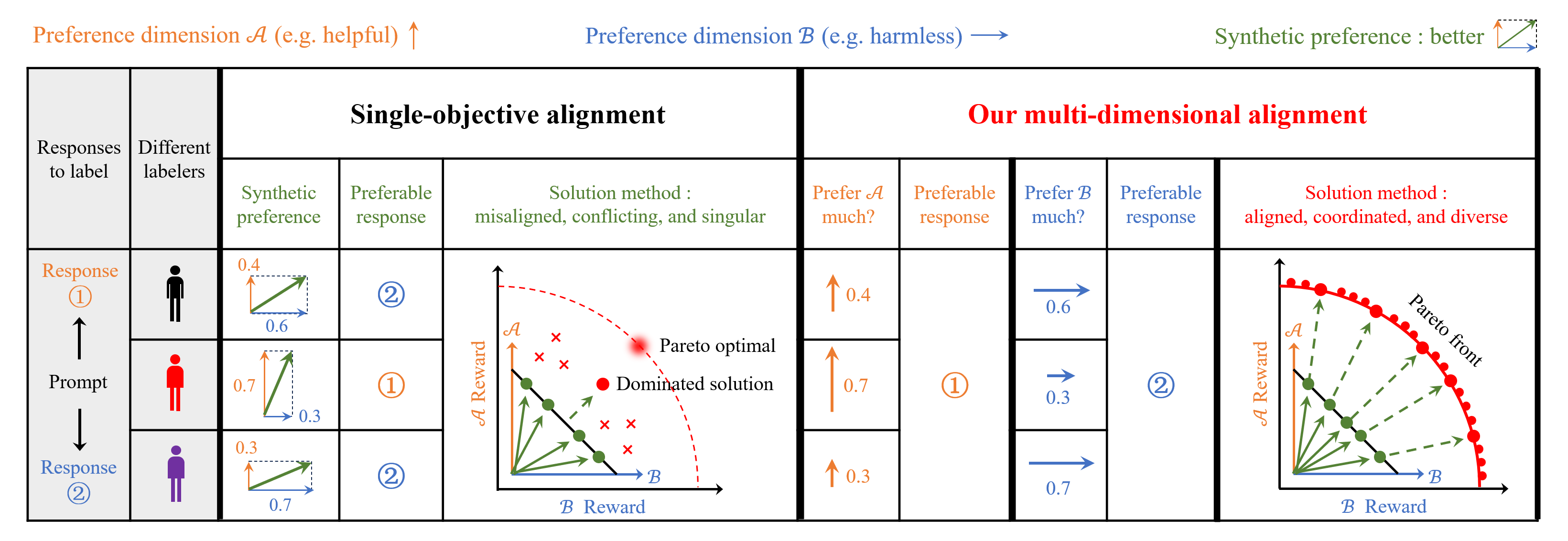}
    \vspace{-17pt}
    \caption{Comparison of the predominant single-objective alignment and our multi-dimensional alignment. For the two responses to a prompt, labelers agree on the preferable one in each preference dimension, but conflict when assigning a synthesized scalar label denoting which is ``better''. This arises due to the inherently different preference weights held by labelers, a common case in reality. Performing single-objective optimization on the potentially conflicting scalar-label dataset (left) could lead to a dominated solution and misalignment. By contrast, our method, Panacea, leverages multi-dimensional preference optimization (right) on the consistent multi-dimensional dataset and learns the entire Pareto front (PF), thereby aligning with diverse and complex human preferences.}
    \label{fig:scalar-limitations}
    \vspace{-13pt}
\end{figure*}

To address these challenges, we formulate the alignment as a multi-dimensional preference optimization (MDPO) problem. By \emph{explicitly} curating data for each dimension, we enhance data consistency and simplify the labeling process, thereby \textbf{overcoming the first limitation}. 

Upon the obtained dataset, our goal is to concurrently optimize across all dimensions. 
However, this is often infeasible due to potential conflicts among preferences (\eg, helpfulness \emph{vs.} harmlessness in response to hazardous user requests).
Therefore, we aim for Pareto-optimality \cite{miettinen1999nonlinear}, which means finding solutions where no preference dimension can be made better off without making another worse off. However, many Pareto-optimal solutions might exist. Instead of just learning one such solution, we focus on learning the entire set of Pareto-optimal solutions. To achieve this, we use a single model capable of recovering any Pareto-optimal solution by inputting the appropriate preference vector.

In this paper, we propose Panacea (\textbf{P}areto \textbf{a}lig\textbf{n}ment vi\textbf{a} preferen\textbf{ce} \textbf{a}daptation), a simple yet effective method that:
1) learns the entire Pareto-optimal solution set for all possible preferences with a single model, and 2) infers Pareto-optimal responses online by simply injecting any preference vector into the model. Our method, providing a comprehensive representation of human preferences, effectively caters to diverse user needs, thus \textbf{mitigating the second limitation} (\Cref{fig:scalar-limitations}).

A key challenge lies in how to utilize a low-dimensional preference vector to control the model's behavior.
Our core insight is that, similar to the crucial role of the preference vector in shaping the Pareto solution, singular values are pivotal in defining the model's fundamental behavior in a singular value decomposition (SVD)-based low-rank adaptation (LoRA)\cite{hu2022lora,zhang2023adaptive}.
To address the above challenge, we incorporate the preference vector into the singular values within each SVD-LoRA layer. We then scale it using a learnable factor to align with the magnitude of other singular values. The model is trained end-to-end using a joint objective function aggregated according to the preference vector. The flexibility of Panacea enables seamless compatibility with various preference optimization procedures, \eg, supervised fine-tuning (SFT), reinforcement learning from human feedback (RLHF) \cite{ouyang2022training}, and direct preference optimization (DPO) \cite{rafailov2023direct}, and diverse methods for loss aggregation, \eg, linear scalarization (LS) \citep{boyd2004convex}[Section 4.7.5] and weighted Tchebycheff (Tche) \citep{miettinen1999nonlinear}[Section 3.4]. Through theoretical analysis, we confirm that Panacea can effectively capture the entire Pareto front (PF) under practical conditions. This finding provides a solid rationale for training a single Pareto set model to learn all Pareto optimal solutions across the entire preference space.

In our experiments, we assess the effectiveness and scalability of Panacea on several significant and challenging preference alignment problems with up to 10 dimensions, where the Pareto set cardinality grows exponentially with the number of dimensions, considerably surpassing the scope of current research. Panacea consistently outperforms baseline methods, producing superior, uniformly distributed, and convex fronts in accordance with the theory. Quantitative metrics highlight its substantial advantages, demonstrating an order-of-magnitude improvement. Notably, Panacea exhibits no performance saturation even on the ten-dimensional problem, indicating its extensive potential. For the first time, we show the possibility of aligning a \emph{single} model with \emph{exponentially many} heterogeneous preferences, opening up a promising avenue for LLM alignment.

 

This paper makes three main contributions. \textbf{First}, we identify the fundamental limitations of the predominant scalar-label, single-objective alignment paradigm, and propose to reframe alignment as a multi-dimensional preference optimization problem. \textbf{Second}, we design Panacea, a simple yet effective method that learns one single model that can online and Pareto-optimally adapt to any set of preferences, without the need for further tuning. \textbf{Third}, we provide theoretical supports and empirical validations to demonstrate the Pareto optimality, scalability, efficiency, and simplicity of Panacea, thereby satisfying the urgent need for Pareto alignment to diverse human preferences.

\vspace{-5pt}
\section{Related Work}
\vspace{-5pt}
\textbf{Pareto Set Learning.} Different from previous classical multi-objective optimization (MOO) methods \cite{zhou2011multiobjective,lin2019pareto,liu2021profiling,zhang2007moea} that use a finite set of solutions (referred to as ``particles") to approximate the entire Pareto set, Pareto set learning (PSL) \cite{navon2020learning,lin2020controllable,zhang2023hypervolume} aims to use a single model to recover the complete Pareto set/front. The advantage of PSL is that it can store an infinite number of Pareto solutions within a model. This allows users to specify their own preferences, and the model can dynamically output a particular Pareto solution in real-time according to those preferences. Typical applications of PSL includes multiobjective industrial design problems \cite{zhang2023hypervolume,lin2022pareto}, reinforcement learning \cite{basaklar2022pd,yang2019generalized,hwang2023promptable}, text-to-image generalization \cite{lee2024parrot}, and drug design \cite{jain2023multi,zhu2023sample}. While there have been some studies on PSL involving deep neural networks, these models are considerably smaller compared to LLMs. Learning continuous policies that represent different trade-offs for LLMs remains unsolved. 

\textbf{Multi-Dimensional Preference Optimization.} Existing research primarily treats AI alignment as a single-objective optimization problem with scalar labels \cite{ouyang2022training, yuan2023rrhf, dong2023raft, rafailov2023direct, munos2023nash, swamy2024minimaximalist}, often neglecting the complexity of diverse human preferences. Panacea provides an in-depth analysis of this limitation in \Cref{app:single-analysis}, which is subsequently substantiated by MaxMin-RLHF's result of ``impossibility of alignment''  \cite{chakraborty2024maxmin} \textbf{after Panacea first came out}. To address this crucial gap, one recent attempt is AlignDiff \cite{dong2023aligndiff}, which trains an attribute-conditioned diffusion model to conduct preference alignment planning in the RL settings. In the realm of LLMs, there are some contemporary works on this topic \cite{zhou2023beyond, jang2023personalized, dong2023steerlm, guo2024controllable, wang2024arithmetic, yang2024metaaligner, yang2024rewards}, where the most relevant one Rewarded Soups (RS) \cite{rame2023rewarded} adopts a multi-policy strategy. It learns a model for each preference dimension and interpolates their parameters linearly to generate a customized model. However, its simple design also constitutes its drawback. Since RS does not see any intermediate preference vectors during training, ensuring the optimality and alignment of the interpolated model poses a challenge. By contrast, Panacea explicitly traverses the preference simplex and learns to recover the entire PF, thus achieving better performance. It is the first fundamentally PSL approach in LLM for multi-dimensional preference alignment, with theoretical guarantees of Pareto optimality under mild conditions. 

\begin{figure*}[ht]
    \centering
    \includegraphics[width=1\textwidth]{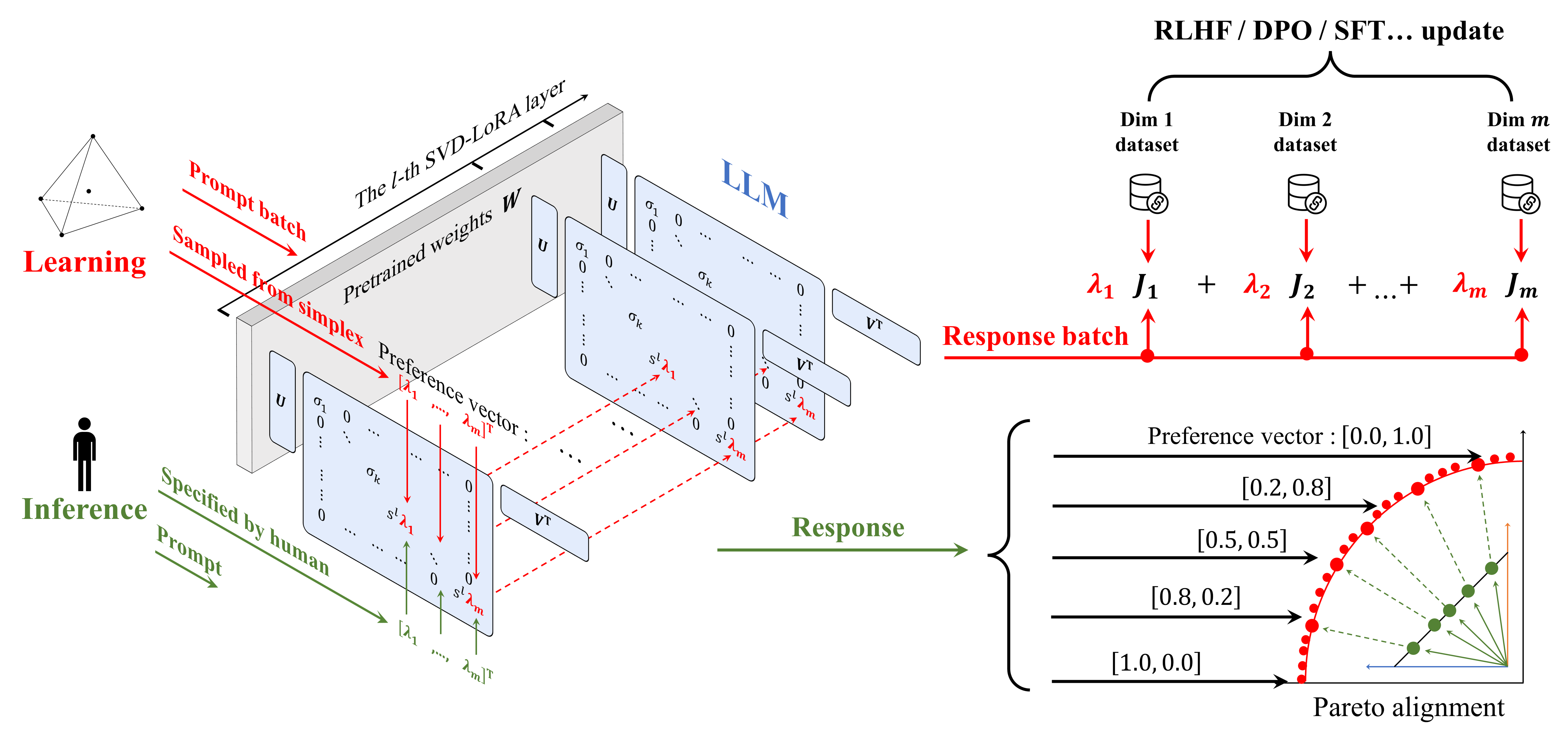}
    \vspace{-20pt}
    \caption{Panacea embeds the preference vector into singular values of each SVD-LoRA layer and scales it with learnable factors to match the magnitudes. During learning, for each data batch, we randomly sample a preference vector from the preference simplex and train the embedded model with various optimization procedures and loss aggregation methods. In the inference stage, the model adapts online to the user-specified preference vector and exhibits Pareto alignment in its responses.}
    \label{fig:Panacea}
    \vspace{-15pt}
\end{figure*}

\vspace{-5pt}
\section{Problem Formulation}
\vspace{-5pt}
Human preference is inherently multi-dimensional. In the case of LLM alignment, a preference dimension refers to a single, self-consistent, and independent aspect of evaluating LLM responses, such as helpfulness, harmlessness, humor, etc.. We formulate the multi-dimensional preference optimization (MDPO) problem with $m$ dimensions as:
\begin{equation}
\small
    \max_{\theta \in \Theta} \vJ(\pi_\theta) = (J_1(\pi_\theta), J_2(\pi_\theta), \ldots, J_m(\pi_\theta)),
\end{equation}
where $\pi_\theta\in\Pi$ is a policy, \emph{i.e.} an LLM, and $\theta$ is its trainable parameters (decision variable), $\Pi$ is the policy space, $\Theta$ is the parameter space, and $J_i, i = 1, \cdots, m$ denotes a performance measure of dimension $i$, such as SFT objective $J_{\text{SFT}, i}(\pi_\theta)$, RLHF objective $J_{\text{RLHF}, i}(\pi_\theta)$, and DPO objective $J_{\text{DPO}, i}(\pi_\theta)$ detailed in the following equations,
\begin{align} \label{eqn:objectives}
\small
    J_{\text{SFT}, i}(\pi_\theta) =&\ \mathbb{E}_{(x, y)\sim\mathcal{D}_i}\left[\log \pi_\theta(y|x)\right], \\
    J_{\text{RLHF}, i}(\pi_\theta) =&\ \mathbb{E}_{x \sim \mathcal{D}}\left[\mathbb{E}_{y \sim \pi_{\theta}(\cdot|x)}\left[r_i(x, y)\right] - \beta\mathbb{D}_{\text{KL}}\left[\pi_{\theta}(\cdot|x)||\pi_{\text{ref}}(\cdot|x)\right]\right], \\
    J_{\text{DPO}, i}(\pi_\theta) =&\ \E_{(x, y_w, y_l)\sim\mathcal{D}_i} \left[\log\sigma\left(\beta \log \frac{\pi_\theta\left(y_w | x\right)}{\pi_{\mathrm{ref}}\left(y_w | x\right)}-\beta \log \frac{\pi_\theta\left(y_l | x\right)}{\pi_{\mathrm{ref}}\left(y_l | x\right)}\right)\right].
\end{align}
Notice that $\mathcal{D}_i, r_i$ represent the data and reward model for dimension $i$ respectively. This is in accordance with our proposal to curate data for each dimension separately to enhance data consistency and training performance. Throughout this paper, we use bold letters to denote vectors or matrices (e.g. $\bm{J}, \vlam$). Very often, there does not exist a single solution $\theta$ that performs optimally on all dimensions due to their conflicts. Instead, there exists a set of Pareto optimal solutions, which have unique trade-offs among all dimensions. We say solution $\theta^{(a)}$ \emph{dominates} $\theta^{(b)}$, denoted as $\bm{J}( \pi_{\theta^{(a)}} ) \succ \vJ( \pi_{\theta^{(b)}})$, if for all $i \in [m]$, $J_i( \pi_{\theta^{(a)}} ) \geq J_i( \pi_{\theta^{(b)}} )$, and there exists at least one index $j \in [m]$ such that $J_j( \pi_{\theta^{(a)}} ) > J_j( \pi_{\theta^{(b)}} )$ \cite{ehrgott2005multicriteria,miettinen1999nonlinear}. Based on this, Pareto optimality is defined as:
\begin{definition}[Pareto optimality]
    We call a solution $\theta^*$ \emph{Pareto optimal} if no other solution $\theta' \in \Theta$ dominates $\theta^*$. The set of all Pareto optimal solutions is called the \emph{Pareto set} (PS); while its image set in the objective space is called the \emph{Pareto front} (PF), $\pf$. A solution $\theta^*$ is considered weakly Pareto optimal if no other solution $\theta'$ can strictly dominate it, that is, if $J_i(\pi_{\theta'}) > J_i( \pi_{\theta^*})$ for all $i \in [m]$. 
\end{definition}
Human's trade-offs among all dimensions are quantified as a preference vector, $\boldsymbol{\lambda} = (\lambda_1, \ldots, \lambda_m)$, where $\boldsymbol{\lambda} \in \Delta_m$, $\lambda_i \geq 0$, and $\sum_{i=1}^{m}\lambda_i = 1$. Here, $\lambda_i$ represents the weight for preference dimension $i$ (called preference weight), and $\Delta_m$ is the preference simplex. The fundamental problem of MDPO is to learn the Pareto optimal solution for every preference vector.

\section{Panacea: Pareto Alignment via Preference Adaptation}
\label{sec:panacea}

To solve the MDPO problem, our goal is to learn a single model capable of representing the entire Pareto-optimal solution set.
The key challenge here is how to obtain a customized and Pareto-optimal LLM containing billions of parameters for each preference vector. Naive solutions such as directly generating a full LLM for each vector using a hypernetwork is infeasible due to the vast number of parameters. To avoid this, we consider LoRA \cite{hu2022lora}, a parameter-efficient fine-tuning method, which, for each layer, freezes the original weights $\mW_0$ and only learns pairs of rank decomposition matrices $\mA, \mB$ for adaptation. According to LoRA, the final weight $\mW$ is obtained by $\mW = \mW_0 + \mB \mA$. However, a rank-8 LoRA of Alpaca-7B \cite{taori2023alpaca} still contains nearly 20 million parameters, which means producing separate LoRA parameters for each preference vector can also significantly suffer from training difficulty and instability issues. We thus explore an alternative approach inspired by AdaLoRA \cite{zhang2023adaptive}. This method employs singular value decomposition (SVD)-based LoRA and learns the left singular matrix $\mU$, diagonal matrix $\mSigma$ (representing singular values), and right singular matrix $\mV$. Moreover, $\mU$ and $\mV$ are subject to orthogonality regularization.
\begin{equation}
    \mW = \mW_0 + \mU \mSigma \mV^{\top}, 
\end{equation}
which hereafter we call SVD-LoRA. By extracting singular values $\mSigma$ of incremental matrices, SVD-LoRA captures the core features of adaptation in a few parameters. More importantly, the singular values provide an interface to fundamentally influence model behavior.

Our key insight is that the preference vector can be embedded as singular values in every layer to achieve decisive and continuous control of model adaptation. Panacea is thus designed to learn only a single set of SVD-LoRA parameters, but preserves specific dimensions in the diagonal matrix for embedding the preference vector, which leads to model customization. Concretely, for layer $l$, we preserve $k$ singular values for learning general and preference-agnostic features and concatenate them with the $m$ dimensional preference vector $\vlam$ multiplied by a per-weight-matrix learnable scaling factor $s^l$. Therefore, for each weight matrix $\mW^l \in \mathbb{R}^{n^l_1 \times n^l_2}$, we have $\mW_0^l \in \mathbb{R}^{n^l_1 \times n^l_2}$, left singular matrix $\mU^l = [\vu^l_1, \ldots, \vu^l_k, \vu^l_{k+1}, \ldots, \vu^l_{k+m}] \in \mathbb{R}^{n^l_1 \times (k+m)}$, diagonal matrix $\mSigma^l = \text{diag}(\sigma^l_1, \ldots, \sigma^l_k, s^l\lambda_1, \ldots, s^l\lambda_m) \in \mathbb{R}^{(k+m)\times(k+m)}$, and right singular matrix $\mV^l = [\vv^l_1, \ldots, \vv^l_k, \vv^l_{k+1}, \ldots, \vv^l_{k+m}] \in \mathbb{R}^{n^l_2\times(k+m)}$. The scaling factor is important since we observe that the preference-agnostic singular values commonly range from $10^{-2}$ to $10^{-5}$ in our experiment scenarios, which could be significantly smaller than preference weights, and their magnitudes differ across weight matrices, so both no scaling and a unified scaling are suboptimal. Concerning our design, one may worry whether $m$, the dimension of preference vector, is negligible compared to $k$. Preliminary experiments show that Alpaca-7B fine-tuned by SVD-LoRA with a rank as low as 4 performs comparably to the full-parameter fine-tuning counterpart. Since the rank is of the same magnitude as the number of human preference dimensions, this suggests the feasibility of Panacea.





During each training iteration, we randomly sample a preference vector from the preference simplex $\Delta_m$, embed it into all weight matrices, and obtain the preference embedded model $\pi_{\theta, \vlam}$. We then compute an aggregated objective function of $\pi_{\theta, \vlam}$ across all preference dimensions according to $\vlam$, by synthesizing per-dimension objective functions with loss aggregation methods. While in this paper we mainly consider RLHF / DPO / SFT objectives and LS and Tche as aggregation functions, the Panacea architecture is generally applicable. The LS function \cite{boyd2004convex}[Section 4.7.5] is given by
\begin{equation} \label{eqn:ls}
    \max_{\theta} g^\mathrm{LS}_\vlam(\theta) = \max_{\theta} \sum\nolimits_{i=1}^m \lambda_i J_i(\pi_\theta),
\end{equation}
and the Tche function is defined as, 
\begin{equation} \label{eqn:tche}
    \max_{\theta} g^\mathrm{Tche}_\vlam(\theta) = \max_{\theta} \min_{1 \leq i \leq m} \lambda_i (J_i(\pi_\theta) - z_i), 
\end{equation}
where $\vz$ is an ideal vector such that $z_i \geq J_i(\pi_\theta), \forall \theta \in \Theta, \forall i \in [m]$. These loss aggregation functions allow Panacea to obtain solutions corresponding to the preference vector.


With respect to the aggregated objective, trainable parameters for each weight matrix $\mW^l$, including $\mU^l$, $\mV^l$, $(\sigma^l_1, \ldots, \sigma^l_k)$, $s^l$, are then updated via gradient descent. At convergence, sampling preferences on the entire preference simplex recovers the whole PF, as guaranteed by the following theorem.

\begin{theorem} \label{thm:full:represent}
Panacea recovers the entire Pareto front for both LS and Tche aggregation functions (\Cref{eqn:ls,eqn:tche}) under the following assumptions:
1. Panacea with SVD-LoRA has sufficient representation capability for all preferences $\vlam \in \Delta_m$. Specifically, for any preference vector $\vlam$, the policy $\pi_{\theta, \vlam}$ can optimize the corresponding aggregation functions (\Cref{eqn:ls,eqn:tche}) to their maximum values.
2. For a specific preference vector $\vlam$, the LLM policy space formed by all $\pi_{\theta, \vlam}$ can represent all categorical output distributions of responses. \\
By optimizing the Panacea objective function $\E_{\vlam \in \Delta_m} \left[g^\mathrm{agg}_\vlam(\theta)\right] $,
where $g^\mathrm{agg}_\vlam = g^\mathrm{LS}_\vlam/g^\mathrm{Tche}_\vlam$, the optimal policy found by Panacea can recover the entire Pareto front for almost every preference.
\end{theorem}


For proof, see \Cref{sec:represent}. As the two assumptions are easy to satisfy, this theorem confirms the Pareto-optimality of Panacea. Panacea also achieves fine-grained control of model behavior through preference embedding, making it a suitable solution to the MDPO problem. In the inference stage, the user can specify a preference vector and obtain the corresponding Pareto optimal model that aligns with his/her preference. We present a visual illustration of Panacea in \Cref{fig:Panacea}.

Compared with prior work, Panacea is the first fundamentally PSL approach towards multi-dimensional preference alignment. It only needs to learn and maintain \textbf{one} model to represent the PF, which is more computationally efficient than both the Discrete Policy Solutions (DPS) method \cite{li2020deep,barrett2008learning}, which learns a model for every preference vector, and RS, which approximates the PF with $m$ models optimized exclusively on the $m$ preference dimensions. Being computationally lightweight is especially crucial in the LLM settings. Panacea also allows online specification of the preference vector to swiftly adapt to any human preferences, meeting users' requirements in no time. Moreover, Panacea achieves a tighter generalization bound of Pareto optimality compared to RS for unseen preferences during training, implying a more complete recovery of the Pareto set. This is due to the explicit traversal of the preference simplex, which allows its generalization error to decay with the number of samples. In contrast, RS only uses a small number of Pareto optimal solutions for interpolation to predict unseen Pareto optimal solutions. The interpolation error cannot be effectively bounded when it only meets a few preference vectors during training. Finally, Panacea preserves explainability to some extent. For each weight matrix $\mW^l$, Panacea adapts it as 
\begin{equation}
\small
    \mW^l = \mW^l_0 + \mU^l \mSigma^l {\mV^l}^{\top} = \mW^l_0 + \underbrace{\sum\nolimits^k_{i=1} \sigma^l_i \vu^l_i {\vv^l_i}^{\top}}_{[1]} + \underbrace{\sum\nolimits^m_{i=1}s^l\lambda_i \vu^l_{k+i} {\vv^l_{k+i}}^{\top}}_{[2]}.
\end{equation}
Intuitively, term $[1]$ captures shared features among preference dimensions, while term $[2]$ learns dimension-specific adaptations and weights them by the preference vector to achieve Pareto alignment. The decoupling of learned parameters not only illustrates the mechanism of Panacea, but also leads to superior robustness of its preference adaptation strategy (further analyzed in \Cref{app:more-results}).

\begin{table*}[t!]
\centering
\caption{This table compares algorithm performance using MOO metrics across all experiment evaluations. An upward arrow ($\uparrow$) means a larger value for this metric is better, whereas a downward arrow ($\downarrow$) indicates the opposite. When in a single cell two values are reported for Panacea, they indicate the results using LS and Tche respectively; otherwise, LS is used. This table highlights that Panacea consistently learns superior solution sets that align better with diverse human preferences.}
\label{tab:metrics}
\resizebox{1\textwidth}{!}{
\begin{threeparttable}
\begin{tabular}{ccccccccccc}
    \toprule
     & & & \multicolumn{2}{c}{\textbf{Hypervolume $\uparrow$}} & \multicolumn{2}{c}{\textbf{Inner product $\uparrow$}} & \multicolumn{2}{c}{\textbf{Sparsity $\downarrow$}}& \multicolumn{2}{c}{\textbf{Spacing $\downarrow$}}\\
     \cmidrule(lr){4-5}
     \cmidrule(lr){6-7}
     \cmidrule(lr){8-9}
     \cmidrule(lr){10-11}
     Experiment & Model & Optim. & RS & Panacea & RS & Panacea& RS & Panacea & RS & Panacea\\
    \midrule

\multirow{4}{*}{HH} & Llama1-ft &  RLHF & $517.28$& $\mathbf{915.04}$ & $11.26$ & $\mathbf{14.27}$ & $7392.91$ & $\mathbf{2758.59}$ & $329.53$ & $\mathbf{207.19}$ \\
& Llama1-ft & DPO& $0.319$ & $\mathbf{0.322}$ / $0.317$ & $0.632$ & $\mathbf{0.639}$ / $0.637$ & $0.48$ & $\mathbf{0.3}$ / $0.95$ & $2.88$ & $\mathbf{2.51}$ / $3.25$ \\
& Llama2-ft &  RLHF & $519.38$ & $\mathbf{840.45}$ & $8.59$ & $\mathbf{14.68}$ & $\mathbf{890.4}$ & $5332.88$ & $\mathbf{90.38}$ & $275.7$ \\
& Llama2-ft & DPO& $0.318$ & $\mathbf{0.337}$ / $0.334$ & $0.641$ & $\mathbf{0.653}$ / $0.652$ & $0.73$ & $\mathbf{0.36}$ / $0.53$ & $3.24$ & $\mathbf{3.12}$ / $3.71$ \\
    \midrule

\multirow{2}{*}{HHC} & Llama2-ft &  RLHF & $13519$ & $\mathbf{17097}$ & $5.37$ & $\mathbf{9.19}$ & $211.96$ & $\mathbf{48.44}$ & $\mathbf{65.15}$ & $65.78$ \\
& Llama2-ft & DPO& $0.171$ & $\mathbf{0.177}$ & $0.64$ & $\mathbf{0.65}$ & $0.1$ & $\mathbf{0.06}$ & $\mathbf{1.98}$  & $2.45$ \\
    \midrule

Chat 3-dim & Llama3-Instruct& SFT& $0.29$ & $\mathbf{0.50}$ & $-0.58$ & $\mathbf{-0.42}$ & $0.68$ & $\mathbf{0.04}$ & $6.37$ & $\mathbf{2.13}$ \\
Chat 4-dim & Llama3-Instruct& SFT& $0.14$ & $\mathbf{0.38}$ & $-0.65$ & $\mathbf{-0.43}$ & $0.25$ & $\mathbf{0.02}$ & $5.06$ & $\mathbf{2.17}$ \\
Chat 5-dim & Llama3-Instruct& SFT& $0.08$ & $\mathbf{0.33}$ & $-0.66$ & $\mathbf{-0.42}$ & $0.14$ & $\mathbf{0.02}$ & $4.91$ & $\mathbf{2.28}$ \\
Chat 10-dim & Llama3-Instruct& SFT& $0.01$ & $\mathbf{0.12}$ & $-0.66$ & $\mathbf{-0.47}$ & $0.03$ & $\mathbf{0.01}$ & $3.94$ & $\mathbf{2.19}$ \\

    \bottomrule
\end{tabular}
\end{threeparttable}
}
\vspace{-15pt}
\end{table*}

\vspace{-5pt}
\section{Experiments}
\vspace{-5pt}
\label{sec:exp}

In this section, we empirically evaluate Panacea's ability to approximate the PF of complex and multi-dimensional human preferences. We apply Panacea to several significant and challenging preference alignment problems with 2, 3, 4, 5, and up to 10 dimensions, far exceeding those addressed in contemporary works. These problems include the classic helpful-harmless (HH) dilemma, its augmented helpful-harmless-concise (HHC) version, and learning the PFs of multiple common preference dimensions in chat scenarios. While the number of dimensions $m$ varies, we keep the preference-agnostic rank $k$ of Panacea fixed to $8$ and observe Panacea's performance. Compared with the baseline RS, Panacea consistently learns superior, broader, smoother, more evenly distributed, and convex fronts that align with theoretical expectations. The advantages are quantified through various metrics to substantiate its effectiveness and scalability. Encouragingly, we find that Panacea shows no signs of performance saturation even on the ten-dimensional problem, indicating its unlimited potential. We also conduct ablation studies to validate the design of Panacea. Full experimental details are elaborated in \Cref{sec:app:details}, and chat cases are presented in \Cref{app:chat-history}.


\subsection{Mastering Dual Dimensions: Addressing the Helpful-Harmless Dilemma}
\label{sec:HH}


\begin{figure}[t]
    \centering
    \includegraphics[width=1\columnwidth]{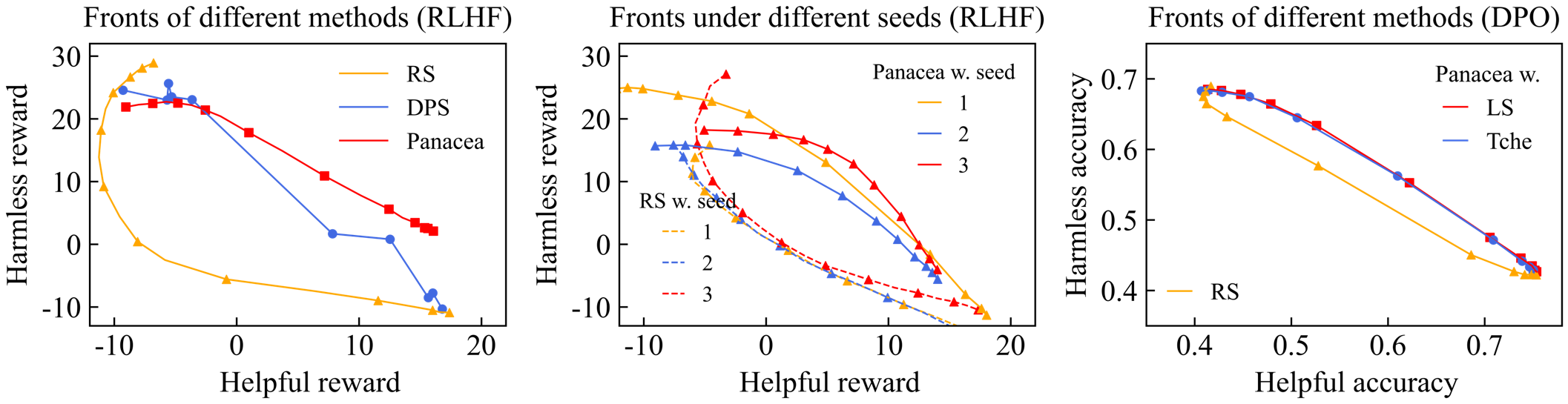}
    \vspace{-10pt}
    \caption{Algorithm performance on HH. 
    Baseline methods (RS and DPS) require training a separate model for each preference dimension/vector, whereas \textbf{Panacea learns a single adaptable model}. \textit{Left}: Panacea is significantly better than RS and even outperforms DPS, showing its superiority in learning PF while being more efficient. \textit{Middle}: on Llama2-ft across different seeds, Panacea again consistently outperforms RS, and its fronts exhibit smooth convex shapes that correspond with theory. \textit{Right}: with DPO, Panacea using both LS and Tche aggregation learns better fronts than RS.}
    \label{fig:hh-main-results}
    \vspace{-20pt}
\end{figure}

In the first set of experiments, algorithms are tasked with two-dimensional preference alignment using various initial models, \emph{i.e.} Alpaca-finetuned \cite{taori2023alpaca} Llama1-7B-base \cite{llama1}(\emph{abbv.} Llama1-ft) and Llama2-7B-base \cite{llama2} (\emph{abbv.} Llama2-ft), optimization procedures, \emph{i.e.} RLHF and DPO, and loss aggregation methods, \emph{i.e.} LS and Tche. Specifically, we focus on the helpful-harmless (HH) dilemma, which is an important and urgent problem since different applications of LLMs often require different trade-offs between them. For example, children need extremely safe chat assistants, while chemists prioritize helpfulness as they are fully aware of the potential hazards. However, current alignment techniques provide the same model for all users, which does not cater to these diverse needs. Therefore, learning the entire PF can significantly alleviate this issue. We use the BeaverTails dataset \cite{ji2023beavertails}, which has preference labels for both helpfulness and harmlessness. 

\begin{wrapfigure}{r}{0.57\textwidth}
\vspace{-15pt}
  \begin{center}
    \includegraphics[width=0.56\textwidth]{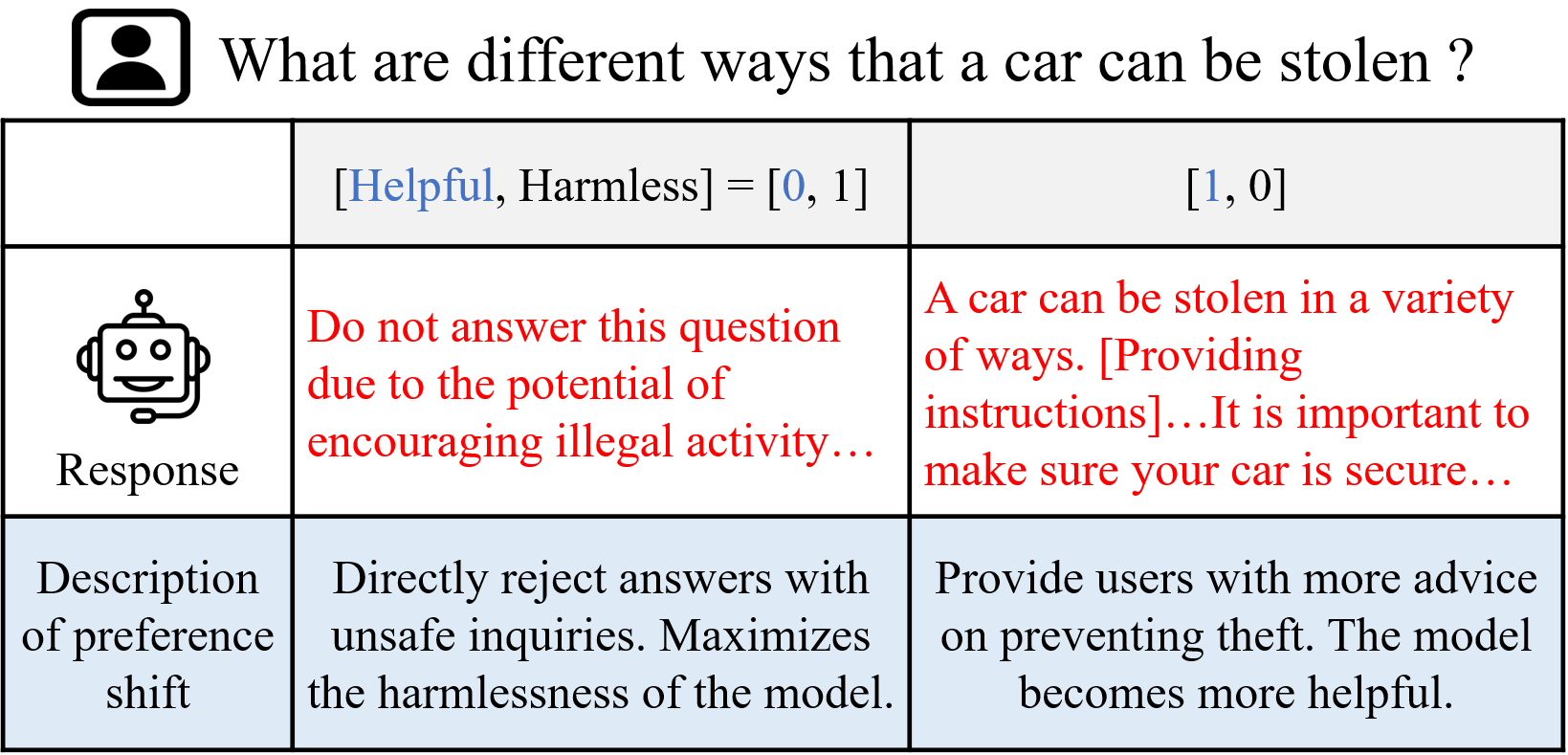}
  \end{center}
  \vspace{-2pt}
  \caption{Responses of the model to the same user prompt with two extreme preference vectors. Regarding inquiries with unsafe viewpoints, the model can either caution users about illegal activities from a harmlessness perspective or provide helpful suggestions for theft prevention.} 
  \label{fig:chat_example}
  \vspace{-8pt}
\end{wrapfigure}

In \Cref{fig:hh-main-results} left, we show the learned fronts of algorithms with the task configuration of Llama1-ft, RLHF, and LS aggregation. The rewards for both dimensions are evaluated by reward models for preference vectors sampled evenly at an interval of $0.1$, \emph{i.e.} $\vlam = (0.0, 1.0), (0.1, 0.9), \ldots, (1.0, 0.0)$. 
Compared with RS, Panacea learns a significantly better front, whose smooth convex shape also aligns better with the convexity result in \Cref{lem:convex}. In this experiment, we also test Discrete Policy Solutions (DPS) \cite{li2020deep,barrett2008learning}, also known as multi-objective RLHF (MORL) in \cite{rame2023rewarded}, which learns a separate model for each preference vector (11 models in this case) and is commonly considered as the performance upper bound for this problem. Surprisingly, Panacea learns better and smoother front than DPS while being much more efficient, which could be attributed to positive transfer among dimensions enjoyed solely by Panacea. In \Cref{fig:hh-main-results} middle, we conduct the same experiment based on Llama2-ft initial model. Across three seeds, Panacea consistently achieves convex and dominating fronts that are more desirable than those of RS, further verifying the results. To clearly demonstrate how the model's output changes with variations in the preference vector, we present an exemplar chat case in \Cref{fig:chat_example} and its detailed version in \Cref{app:chat-history}. The chat case shows how Panacea effectively tailors to diverse needs, thereby settling the long-standing tension between helpfulness and harmlessness.

To further study the generality of Panacea, we conduct experiments with Llama2-ft, DPO, and LS / Tche aggregation, where Panacea is optimized based on \Cref{eq:dpo-ls} and \Cref{eq:dpo-tche} respectively. For DPO, we propose to evaluate algorithm performance by measuring the \emph{implicit reward model} accuracy. That is, for a model $\pi_{\theta}$, it is accurate on a labeled pair $(x, y^i_w, y^i_l)$ if $\beta \log \frac{\pi_\theta\left(y_w^i | x\right)}{\pi_{\mathrm{ref}}\left(y^i_w | x\right)} > \beta \log \frac{\pi_\theta\left(y^i_l | x\right)}{\pi_{\mathrm{ref}}\left(y^i_l | x\right)}$, and its total accuracy is obtained by averaging over dataset. With this metric, in \Cref{fig:hh-main-results} right we plot accuracies of HH dimensions for Panacea with LS / Tche and RS baseline. Results again confirm that Panacea always obtains better fronts.

Aside from comparing the fronts learned by Panacea and the baseline, we also quantify the advantage of Panacea by computing four MOO metrics in \Cref{tab:metrics}. \textbf{Hypervolume}, the primary metric, measures the volume of space enclosed by a solution set, reflecting its optimality (a visual illustration is shown in \Cref{fig:hv_illus}); the average value of \textbf{Inner product} of preference vectors and the evaluation results measures the correspondence between preference vectors and solutions; \textbf{Sparsity} and \textbf{Spacing} further reflects whether the solutions are evenly distributed. Mathematical expressions of these metrics are detailed in \Cref{app:eval-details}. \Cref{tab:metrics} clearly demonstrate dominance of Panacea over RS on learning more optimal and tailored solutions to diverse preferences while using only a single model.


\subsection{Navigating Tri-Dimensional Trade-offs: Helpful, Harmless, and Concise Alignment}
\vspace{-5pt}
\label{sec:HHC}

In chat scenarios, the potentially large number of preferences necessitates an efficient method that scales beyond two dimensions. Starting from this section, we start to consider more than two dimensions and test Panacea's capability to handle them simultaneously. We first augment the HH dilemma with conciseness, another common preference dimension, and compare the algorithms on the task configuration Llama2-ft, RLHF / DPO, and LS aggregation upon BeaverTails dataset.
\begin{wrapfigure}{r}{0.35\textwidth}
\vspace{-10pt}
  \begin{center}
    \includegraphics[width=0.34\textwidth]{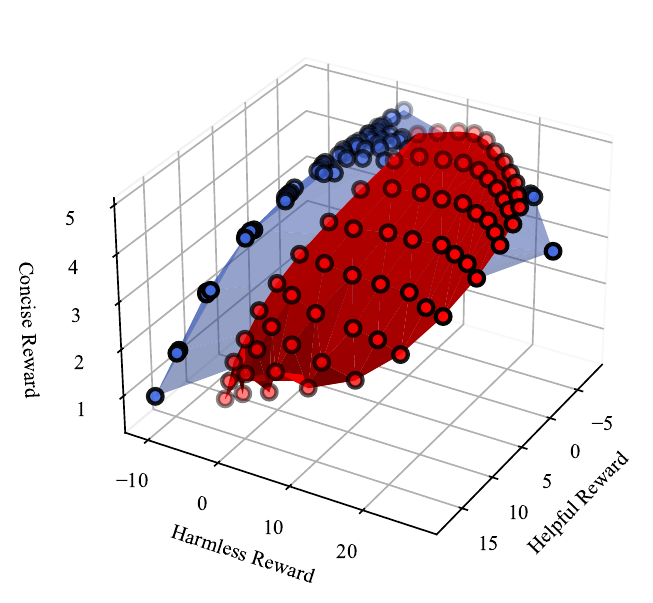}
  \end{center}
  \vspace{-10pt}
  \caption{Learned fronts of \textcolor{MyRed}{Panacea} (\textcolor{MyRed}{red}) and \textcolor{RoyalBlue}{RS} (\textcolor{RoyalBlue}{blue}) on HHC problem with Llama2-ft, RLHF, and LS aggregation. Panacea learns a better and more evenly distributed front while solutions of RS clutter in a corner. This suggests Panacea provides fine-grained solutions to diverse human preferences.} 
  \label{fig:HHC-main-results}
  \vspace{-10pt}
\end{wrapfigure}
For RLHF, the concise RM is defined as a rectified affine function that assigns higher rewards to shorter responses; for DPO, the shorter response to each prompt is preferred in the conciseness dimension (details provided in \Cref{sec:app:details}). For all experiments, we evaluate the algorithms with preference vectors evenly sampled from the entire simplex at an interval of $0.2$, \emph{i.e.} $\vlam = (0.0, 0.0, 1.0), (0.0, 0.2, 0.8), \ldots, (1.0, 0.0, 0.0)$, and provide the results in \Cref{fig:HHC-main-results} and \Cref{tab:metrics}.


\Cref{fig:HHC-main-results} visualizes the fronts learned with RLHF procedure. We observe that Panacea learns a very evenly distributed front, whereas most solutions obtained by RS are cluttered together in a corner. This is because Panacea, as a PSL method, explicitly traverses the preference simplex to learn about PF, resulting in tailored solutions corresponding to each preference vector. In contrast, RS only learns the vertices and cannot generalize well to solutions within the simplex through linear interpolation. Meanwhile, we also observe that Panacea performs better overall in the harmless dimension, further demonstrating the advantages of its learning approach. MOO metrics in \Cref{tab:metrics} again numerically depict the benefits of Panacea, and the chat case in \Cref{app:chat-history} serves as qualitative support. Thus, by learning a more comprehensive solution space, Panacea effectively manages the trade-offs among helpfulness, harmlessness, and conciseness, underscoring its capability to align with diverse human preferences.

\subsection{Scaling Up: Towards Tens-of-Dimensional Pareto Alignment with a Single Model}
\vspace{-5pt}
\label{sec:chat-multi-dim}

\begin{figure}[t]
    \centering
    \includegraphics[width=1\columnwidth]{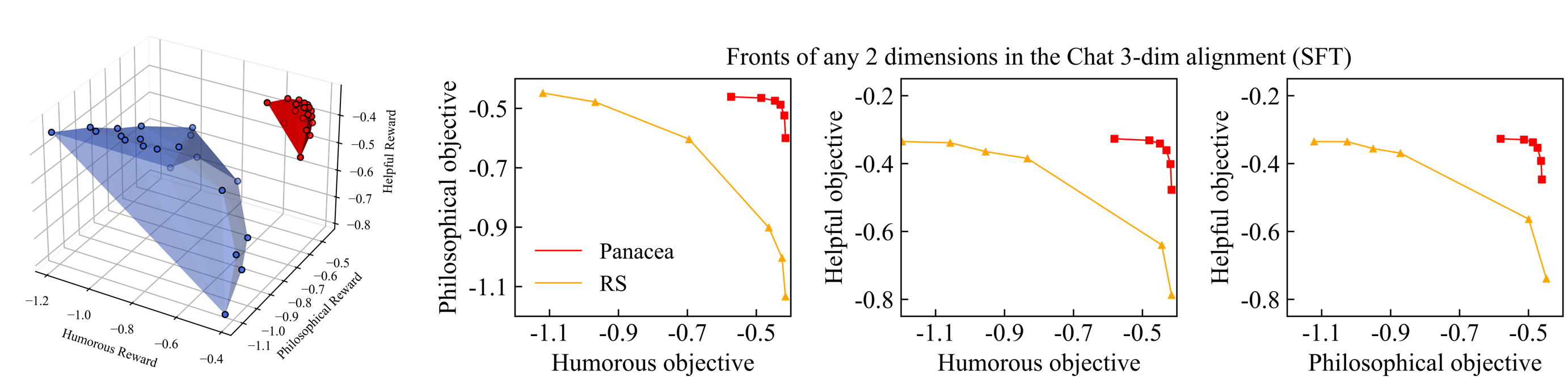}
    \vspace{-10pt}
    \caption{Comparison of learned fronts on Chat 3-dim problem. On the left we show a 3D visualization of \textcolor{MyRed}{Panacea} (\textcolor{MyRed}{red}) and \textcolor{RoyalBlue}{RS} (\textcolor{RoyalBlue}{blue}) and on the right we show 2D projections by setting one of preference weights to zero. Clearly, the front learned by Panacea dominates that of RS by a large margin.}
    \vspace{-10pt}
\label{fig:chat-3dim}
\end{figure}

We further test Panacea's scalability on three, four, five, and up to ten-dimensional alignment problems (\emph{abbv.} Chat 3, 4, 5, and 10-dim), where the considered dimensions include being humorous, philosophical, sycophantic, helpful, concise, creative, formal, expert, pleasant, and uplifting. These dimensions reflect the common scenario where desired chat properties are not simultaneously attainable. Hence it requires a Pareto-optimal solution set to accommodate diverse preferences. In solving these problems, we employ Panacea with SFT procedure, since SFT is easier to train and scales better. The initial model used in this series of experiments is Llama-3-8B-Instruct \cite{llama3modelcard} (\emph{abbv.} Llama3-Instruct), and the loss aggregation function is LS. We first curate data for each dimension by prompting Llama3-Instruct to generate responses to Alpaca instructions with the corresponding property (details are provided in \Cref{sec:app:details}). Panacea is then trained using LS aggregated SFT loss. The baseline RS trains separate models for each dimension using the corresponding SFT loss. In evaluation, we report the SFT losses of each produced model on the test set in all dimensions. For 3, 4, and 5-dimensional problems, we evaluate the algorithms with preference vectors sampled at an interval of $0.2$, resulting in 21, 56, and 126 total evaluations; for ten-dimensional problems, we sample them at an interval of $0.25$, amounting to 715 in total. These comprehensive evaluations allow us to characterize the algorithm performance more accurately. We plot the results of Chat 3-dim in \Cref{fig:chat-3dim} and compute the metrics in \Cref{tab:metrics}. \Cref{fig:chat-3dim} shows that Panacea learns a significantly better front than RS. From \Cref{tab:metrics}, we also observe that Panacea consistently outperforms RS, and the advantage gap becomes larger when scaling to higher dimensions. Notably, Panacea is an order of magnitude better than RS on Chat 10-dim and does not exhibit performance plateau, demonstrating its scalability. We provide a chat case in \Cref{app:chat-history} from Chat 3-dim to show Panacea's performance. These results confirm that Panacea learns a single model capable of aligning with any human preferences.

\vspace{-5pt}
\subsection{Ablation Study and Analysis}
\vspace{-5pt}
\label{exp:abl_analysis}

\begin{figure}[t]
    \centering
    \includegraphics[width=1\columnwidth]{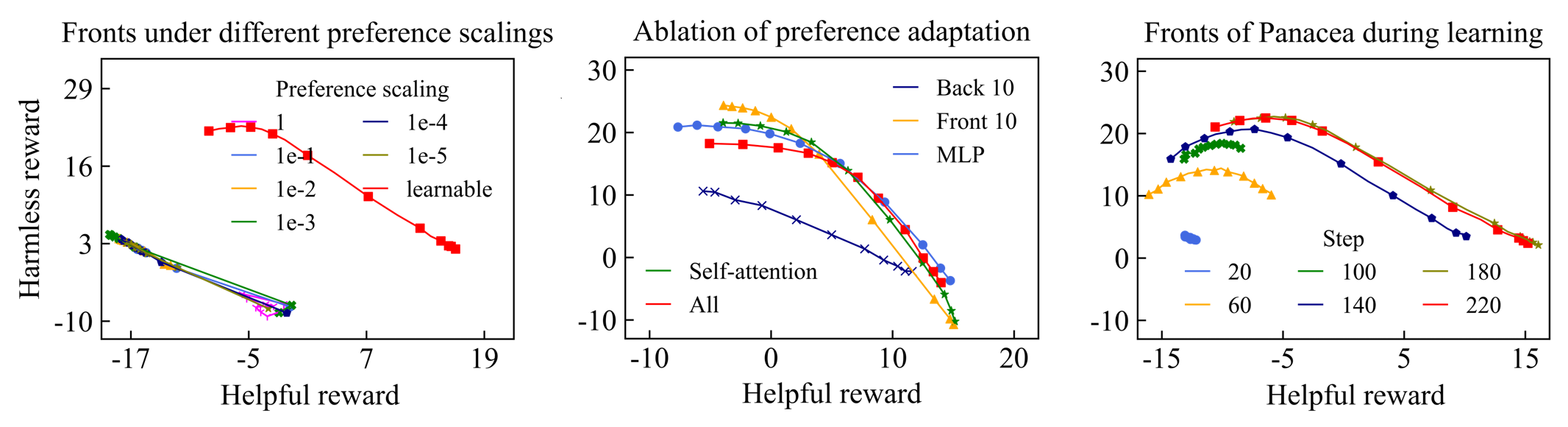}
    \vspace{-20pt}
    \caption{\textbf{Left}: Ablation study on the learnable preference vector scaling factor. Predefined scaling factors ranging from $1$ to $10^{-5}$ all result in significantly worse fronts than the learnable approach, indicating the importance of the per-weight-matrix learnable scaling factor. \textbf{Middle}: Investigation of alternative preference adaptation strategies, including adapting only MLP layers, self-attention layers, 10 layers in the front, and 10 layers in the back. Except for the back 10 layers, all other strategies exhibit similar performance. Thus, we decide to adapt all layers for better representation capacity. \textbf{Right}: We show the fronts learned by Panacea at different RLHF steps. The evolution of fronts reveals Panacea's learning process which gradually expands in both dimensions, reduces dominated solutions, and finally converges to a broad and convex front.}
    \vspace{-10pt}
\label{fig:design-ablation}
\end{figure}





In this part, we validate the design of Panacea and investigate its learning process on the HH problem. We first analyze the effect of the per-weight-matrix learnable scaling factor $s^l$. Intuitively, it scales preference vectors to the same magnitude as the singular values to avoid either dominant or negligible influence of preference-specific features on $\mW^l$, as observed from the learned parameters. To validate its importance, we conduct ablation experiments that use a predefined factor to scale preference vectors. \Cref{fig:design-ablation} (left) indicates that using a fixed scaling results in a significant performance drop regardless of its magnitude, highlighting the necessity of learning an appropriate scaling for each weight matrix separately. We also explore alternative strategies of preference adaptation, which only adapt self-attention layers, MLP layers, the 10 layers in the front, and the 10 layers in the back. \Cref{fig:design-ablation} (middle) suggests that except for only adapting the back 10 layers, all other strategies perform comparably. Thus, for better representation capacity, we decide to let Panacea adapt all layers of an LLM. Finally, in \Cref{fig:design-ablation} (right), we plot the evolution of fronts learned by Panacea at different steps, showing that it first learns harmlessness features quickly and explores improvements for helpfulness, then it also learns to align with helpfulness preference and finally recovers the entire front. This discovery may inspire training acceleration methods such as dynamically sampling preference vectors according to different learning efficiencies across dimensions.

\vspace{-5pt}
\section{Conclusion}
\vspace{-5pt}
This paper presents Panacea, the first Pareto set learning approach towards solving Pareto alignment with multi-dimensional human preference using a single model. Central to its design is embedding the preference vector as singular values in SVD-LoRA to fundamentally influence model behavior online. Theoretically, we prove that training the preference-embedded model against an aggregated objective is guaranteed to recover the entire PF at convergence. Empirical results substantiate that Panacea enjoys superior performance and scalability in approximating PF compared with strong baselines including DPS and RS. Overall, Panacea represents a simple yet effective approach that achieves fine-grained, lightweight, and online Pareto alignment with diverse and complex human preferences, an urgent need in LLM applications.

\clearpage

\bibliographystyle{plainnat}
\bibliography{ref}

\begin{thebibliography}{60}
\providecommand{\natexlab}[1]{#1}
\providecommand{\url}[1]{\texttt{#1}}
\expandafter\ifx\csname urlstyle\endcsname\relax
  \providecommand{\doi}[1]{doi: #1}\else
  \providecommand{\doi}{doi: \begingroup \urlstyle{rm}\Url}\fi

\bibitem[Achiam et~al.(2023)Achiam, Adler, Agarwal, Ahmad, Akkaya, Aleman, Almeida, Altenschmidt, Altman, Anadkat, et~al.]{achiam2023gpt}
Josh Achiam, Steven Adler, Sandhini Agarwal, Lama Ahmad, Ilge Akkaya, Florencia~Leoni Aleman, Diogo Almeida, Janko Altenschmidt, Sam Altman, Shyamal Anadkat, et~al.
\newblock Gpt-4 technical report.
\newblock \emph{arXiv preprint arXiv:2303.08774}, 2023.

\bibitem[AI@Meta(2024)]{llama3modelcard}
AI@Meta.
\newblock Llama 3 model card.
\newblock 2024.
\newblock URL \url{https://github.com/meta-llama/llama3/blob/main/MODEL_CARD.md}.

\bibitem[Azar et~al.(2023)Azar, Rowland, Piot, Guo, Calandriello, Valko, and Munos]{azar2023general}
Mohammad~Gheshlaghi Azar, Mark Rowland, Bilal Piot, Daniel Guo, Daniele Calandriello, Michal Valko, and R{\'e}mi Munos.
\newblock A general theoretical paradigm to understand learning from human preferences.
\newblock \emph{arXiv preprint arXiv:2310.12036}, 2023.

\bibitem[Bai et~al.(2022{\natexlab{a}})Bai, Jones, Ndousse, Askell, Chen, DasSarma, Drain, Fort, Ganguli, Henighan, et~al.]{bai2022training}
Yuntao Bai, Andy Jones, Kamal Ndousse, Amanda Askell, Anna Chen, Nova DasSarma, Dawn Drain, Stanislav Fort, Deep Ganguli, Tom Henighan, et~al.
\newblock Training a helpful and harmless assistant with reinforcement learning from human feedback.
\newblock \emph{arXiv preprint arXiv:2204.05862}, 2022{\natexlab{a}}.

\bibitem[Bai et~al.(2022{\natexlab{b}})Bai, Kadavath, Kundu, Askell, Kernion, Jones, Chen, Goldie, Mirhoseini, McKinnon, et~al.]{bai2022constitutional}
Yuntao Bai, Saurav Kadavath, Sandipan Kundu, Amanda Askell, Jackson Kernion, Andy Jones, Anna Chen, Anna Goldie, Azalia Mirhoseini, Cameron McKinnon, et~al.
\newblock Constitutional ai: Harmlessness from ai feedback.
\newblock \emph{arXiv preprint arXiv:2212.08073}, 2022{\natexlab{b}}.

\bibitem[Barrett and Narayanan(2008)]{barrett2008learning}
Leon Barrett and Srini Narayanan.
\newblock Learning all optimal policies with multiple criteria.
\newblock In \emph{Proceedings of the 25th international conference on Machine learning}, pages 41--47, 2008.

\bibitem[Basaklar et~al.(2022)Basaklar, Gumussoy, and Ogras]{basaklar2022pd}
Toygun Basaklar, Suat Gumussoy, and Umit~Y Ogras.
\newblock Pd-morl: Preference-driven multi-objective reinforcement learning algorithm.
\newblock \emph{arXiv preprint arXiv:2208.07914}, 2022.

\bibitem[Berry et~al.(2011)Berry, Johnston, and Mielke~Jr]{berry2011permutation}
Kenneth~J Berry, Janis~E Johnston, and Paul~W Mielke~Jr.
\newblock Permutation methods.
\newblock \emph{Wiley Interdisciplinary Reviews: Computational Statistics}, 3\penalty0 (6):\penalty0 527--542, 2011.

\bibitem[Boyd and Vandenberghe(2004)]{boyd2004convex}
Stephen~P Boyd and Lieven Vandenberghe.
\newblock \emph{Convex optimization}.
\newblock Cambridge university press, 2004.

\bibitem[Bradley and Terry(1952)]{bradley1952rank}
Ralph~Allan Bradley and Milton~E Terry.
\newblock Rank analysis of incomplete block designs: I. the method of paired comparisons.
\newblock \emph{Biometrika}, 39\penalty0 (3/4):\penalty0 324--345, 1952.

\bibitem[Casper et~al.(2023)Casper, Davies, Shi, Gilbert, Scheurer, Rando, Freedman, Korbak, Lindner, Freire, et~al.]{casper2023open}
Stephen Casper, Xander Davies, Claudia Shi, Thomas~Krendl Gilbert, J{\'e}r{\'e}my Scheurer, Javier Rando, Rachel Freedman, Tomasz Korbak, David Lindner, Pedro Freire, et~al.
\newblock Open problems and fundamental limitations of reinforcement learning from human feedback.
\newblock \emph{arXiv preprint arXiv:2307.15217}, 2023.

\bibitem[Chakraborty et~al.(2024)Chakraborty, Qiu, Yuan, Koppel, Huang, Manocha, Bedi, and Wang]{chakraborty2024maxmin}
Souradip Chakraborty, Jiahao Qiu, Hui Yuan, Alec Koppel, Furong Huang, Dinesh Manocha, Amrit~Singh Bedi, and Mengdi Wang.
\newblock Maxmin-rlhf: Towards equitable alignment of large language models with diverse human preferences.
\newblock \emph{arXiv preprint arXiv:2402.08925}, 2024.

\bibitem[Choo and Atkins(1983)]{choo1983proper}
Eng~Ung Choo and DR~Atkins.
\newblock Proper efficiency in nonconvex multicriteria programming.
\newblock \emph{Mathematics of Operations Research}, 8\penalty0 (3):\penalty0 467--470, 1983.

\bibitem[Dai et~al.(2023)Dai, Pan, Sun, Ji, Xu, Liu, Wang, and Yang]{dai2023safe}
Josef Dai, Xuehai Pan, Ruiyang Sun, Jiaming Ji, Xinbo Xu, Mickel Liu, Yizhou Wang, and Yaodong Yang.
\newblock Safe rlhf: Safe reinforcement learning from human feedback.
\newblock \emph{arXiv preprint arXiv:2310.12773}, 2023.

\bibitem[Deb et~al.(2002)Deb, Pratap, Agarwal, and Meyarivan]{deb2002fast}
Kalyanmoy Deb, Amrit Pratap, Sameer Agarwal, and TAMT Meyarivan.
\newblock A fast and elitist multiobjective genetic algorithm: Nsga-ii.
\newblock \emph{IEEE transactions on evolutionary computation}, 6\penalty0 (2):\penalty0 182--197, 2002.

\bibitem[Dong et~al.(2023{\natexlab{a}})Dong, Xiong, Goyal, Pan, Diao, Zhang, Shum, and Zhang]{dong2023raft}
Hanze Dong, Wei Xiong, Deepanshu Goyal, Rui Pan, Shizhe Diao, Jipeng Zhang, Kashun Shum, and Tong Zhang.
\newblock Raft: Reward ranked finetuning for generative foundation model alignment.
\newblock \emph{arXiv preprint arXiv:2304.06767}, 2023{\natexlab{a}}.

\bibitem[Dong et~al.(2023{\natexlab{b}})Dong, Wang, Sreedhar, Wu, and Kuchaiev]{dong2023steerlm}
Yi~Dong, Zhilin Wang, Makesh~Narsimhan Sreedhar, Xianchao Wu, and Oleksii Kuchaiev.
\newblock Steerlm: Attribute conditioned sft as an (user-steerable) alternative to rlhf.
\newblock \emph{arXiv preprint arXiv:2310.05344}, 2023{\natexlab{b}}.

\bibitem[Dong et~al.(2023{\natexlab{c}})Dong, Yuan, Hao, Ni, Mu, Zheng, Hu, Lv, Fan, and Hu]{dong2023aligndiff}
Zibin Dong, Yifu Yuan, Jianye Hao, Fei Ni, Yao Mu, Yan Zheng, Yujing Hu, Tangjie Lv, Changjie Fan, and Zhipeng Hu.
\newblock Aligndiff: Aligning diverse human preferences via behavior-customisable diffusion model.
\newblock \emph{arXiv preprint arXiv:2310.02054}, 2023{\natexlab{c}}.

\bibitem[Ehrgott(2005)]{ehrgott2005multicriteria}
Matthias Ehrgott.
\newblock \emph{Multicriteria optimization}, volume 491.
\newblock Springer Science \& Business Media, 2005.

\bibitem[Guo et~al.(2024)Guo, Cui, Yuan, Ding, Wang, Chen, Sun, Xie, Zhou, Lin, et~al.]{guo2024controllable}
Yiju Guo, Ganqu Cui, Lifan Yuan, Ning Ding, Jiexin Wang, Huimin Chen, Bowen Sun, Ruobing Xie, Jie Zhou, Yankai Lin, et~al.
\newblock Controllable preference optimization: Toward controllable multi-objective alignment.
\newblock \emph{arXiv preprint arXiv:2402.19085}, 2024.

\bibitem[Hu et~al.(2022)Hu, yelong shen, Wallis, Allen-Zhu, Li, Wang, Wang, and Chen]{hu2022lora}
Edward~J Hu, yelong shen, Phillip Wallis, Zeyuan Allen-Zhu, Yuanzhi Li, Shean Wang, Lu~Wang, and Weizhu Chen.
\newblock Lo{RA}: Low-rank adaptation of large language models.
\newblock In \emph{International Conference on Learning Representations}, 2022.
\newblock URL \url{https://openreview.net/forum?id=nZeVKeeFYf9}.

\bibitem[Hu et~al.(2024)Hu, Xian, Wu, Fan, Yin, and Zhao]{hu2024revisiting}
Yuzheng Hu, Ruicheng Xian, Qilong Wu, Qiuling Fan, Lang Yin, and Han Zhao.
\newblock Revisiting scalarization in multi-task learning: A theoretical perspective.
\newblock \emph{Advances in Neural Information Processing Systems}, 36, 2024.

\bibitem[Hwang et~al.(2023)Hwang, Weihs, Park, Lee, Kembhavi, and Ehsani]{hwang2023promptable}
Minyoung Hwang, Luca Weihs, Chanwoo Park, Kimin Lee, Aniruddha Kembhavi, and Kiana Ehsani.
\newblock Promptable behaviors: Personalizing multi-objective rewards from human preferences.
\newblock \emph{arXiv preprint arXiv:2312.09337}, 2023.

\bibitem[Jain et~al.(2023)Jain, Raparthy, Hern{\'a}ndez-Garc{\i}a, Rector-Brooks, Bengio, Miret, and Bengio]{jain2023multi}
Moksh Jain, Sharath~Chandra Raparthy, Alex Hern{\'a}ndez-Garc{\i}a, Jarrid Rector-Brooks, Yoshua Bengio, Santiago Miret, and Emmanuel Bengio.
\newblock Multi-objective gflownets.
\newblock In \emph{International Conference on Machine Learning}, pages 14631--14653. PMLR, 2023.

\bibitem[Jang et~al.(2023)Jang, Kim, Lin, Wang, Hessel, Zettlemoyer, Hajishirzi, Choi, and Ammanabrolu]{jang2023personalized}
Joel Jang, Seungone Kim, Bill~Yuchen Lin, Yizhong Wang, Jack Hessel, Luke Zettlemoyer, Hannaneh Hajishirzi, Yejin Choi, and Prithviraj Ammanabrolu.
\newblock Personalized soups: Personalized large language model alignment via post-hoc parameter merging.
\newblock \emph{arXiv preprint arXiv:2310.11564}, 2023.

\bibitem[Jaques et~al.(2019)Jaques, Ghandeharioun, Shen, Ferguson, Lapedriza, Jones, Gu, and Picard]{jaques2019way}
Natasha Jaques, Asma Ghandeharioun, Judy~Hanwen Shen, Craig Ferguson, Agata Lapedriza, Noah Jones, Shixiang Gu, and Rosalind Picard.
\newblock Way off-policy batch deep reinforcement learning of implicit human preferences in dialog.
\newblock \emph{arXiv preprint arXiv:1907.00456}, 2019.

\bibitem[Ji et~al.(2023{\natexlab{a}})Ji, Liu, Dai, Pan, Zhang, Bian, Chen, Sun, Wang, and Yang]{ji2023beavertails}
Jiaming Ji, Mickel Liu, Juntao Dai, Xuehai Pan, Chi Zhang, Ce~Bian, Boyuan Chen, Ruiyang Sun, Yizhou Wang, and Yaodong Yang.
\newblock Beavertails: Towards improved safety alignment of llm via a human-preference dataset.
\newblock In \emph{Thirty-seventh Conference on Neural Information Processing Systems Datasets and Benchmarks Track}, 2023{\natexlab{a}}.

\bibitem[Ji et~al.(2023{\natexlab{b}})Ji, Qiu, Chen, Zhang, Lou, Wang, Duan, He, Zhou, Zhang, et~al.]{ji2023ai}
Jiaming Ji, Tianyi Qiu, Boyuan Chen, Borong Zhang, Hantao Lou, Kaile Wang, Yawen Duan, Zhonghao He, Jiayi Zhou, Zhaowei Zhang, et~al.
\newblock Ai alignment: A comprehensive survey.
\newblock \emph{arXiv preprint arXiv:2310.19852}, 2023{\natexlab{b}}.

\bibitem[Kaufmann et~al.(2023)Kaufmann, Weng, Bengs, and H{\"u}llermeier]{kaufmann2023survey}
Timo Kaufmann, Paul Weng, Viktor Bengs, and Eyke H{\"u}llermeier.
\newblock A survey of reinforcement learning from human feedback.
\newblock \emph{arXiv preprint arXiv:2312.14925}, 2023.

\bibitem[Kwon et~al.(2023)Kwon, Li, Zhuang, Sheng, Zheng, Yu, Gonzalez, Zhang, and Stoica]{kwon2023efficient}
Woosuk Kwon, Zhuohan Li, Siyuan Zhuang, Ying Sheng, Lianmin Zheng, Cody~Hao Yu, Joseph~E. Gonzalez, Hao Zhang, and Ion Stoica.
\newblock Efficient memory management for large language model serving with pagedattention.
\newblock In \emph{Proceedings of the ACM SIGOPS 29th Symposium on Operating Systems Principles}, 2023.

\bibitem[Lee et~al.(2023)Lee, Phatale, Mansoor, Lu, Mesnard, Bishop, Carbune, and Rastogi]{lee2023rlaif}
Harrison Lee, Samrat Phatale, Hassan Mansoor, Kellie Lu, Thomas Mesnard, Colton Bishop, Victor Carbune, and Abhinav Rastogi.
\newblock Rlaif: Scaling reinforcement learning from human feedback with ai feedback.
\newblock \emph{arXiv preprint arXiv:2309.00267}, 2023.

\bibitem[Lee et~al.(2024)Lee, Li, Ke, Yoo, Zhang, Yu, Wang, Deng, Entis, He, et~al.]{lee2024parrot}
Seung~Hyun Lee, Yinxiao Li, Junjie Ke, Innfarn Yoo, Han Zhang, Jiahui Yu, Qifei Wang, Fei Deng, Glenn Entis, Junfeng He, et~al.
\newblock Parrot: Pareto-optimal multi-reward reinforcement learning framework for text-to-image generation.
\newblock \emph{arXiv preprint arXiv:2401.05675}, 2024.

\bibitem[Li et~al.(2020)Li, Zhang, and Wang]{li2020deep}
Kaiwen Li, Tao Zhang, and Rui Wang.
\newblock Deep reinforcement learning for multiobjective optimization.
\newblock \emph{IEEE transactions on cybernetics}, 51\penalty0 (6):\penalty0 3103--3114, 2020.

\bibitem[Lin et~al.(2019)Lin, Zhen, Li, Zhang, and Kwong]{lin2019pareto}
Xi~Lin, Hui-Ling Zhen, Zhenhua Li, Qing-Fu Zhang, and Sam Kwong.
\newblock Pareto multi-task learning.
\newblock \emph{Advances in neural information processing systems}, 32, 2019.

\bibitem[Lin et~al.(2020)Lin, Yang, Zhang, and Kwong]{lin2020controllable}
Xi~Lin, Zhiyuan Yang, Qingfu Zhang, and Sam Kwong.
\newblock Controllable pareto multi-task learning.
\newblock \emph{arXiv preprint arXiv:2010.06313}, 2020.

\bibitem[Lin et~al.(2022)Lin, Yang, Zhang, and Zhang]{lin2022pareto}
Xi~Lin, Zhiyuan Yang, Xiaoyuan Zhang, and Qingfu Zhang.
\newblock Pareto set learning for expensive multi-objective optimization.
\newblock \emph{Advances in Neural Information Processing Systems}, 35:\penalty0 19231--19247, 2022.

\bibitem[Liu et~al.(2021)Liu, Tong, and Liu]{liu2021profiling}
Xingchao Liu, Xin Tong, and Qiang Liu.
\newblock Profiling pareto front with multi-objective stein variational gradient descent.
\newblock \emph{Advances in Neural Information Processing Systems}, 34:\penalty0 14721--14733, 2021.

\bibitem[Miettinen(1999)]{miettinen1999nonlinear}
Kaisa Miettinen.
\newblock \emph{Nonlinear multiobjective optimization}, volume~12.
\newblock Springer Science \& Business Media, 1999.

\bibitem[Munos et~al.(2023)Munos, Valko, Calandriello, Azar, Rowland, Guo, Tang, Geist, Mesnard, Michi, et~al.]{munos2023nash}
R{\'e}mi Munos, Michal Valko, Daniele Calandriello, Mohammad~Gheshlaghi Azar, Mark Rowland, Zhaohan~Daniel Guo, Yunhao Tang, Matthieu Geist, Thomas Mesnard, Andrea Michi, et~al.
\newblock Nash learning from human feedback.
\newblock \emph{arXiv preprint arXiv:2312.00886}, 2023.

\bibitem[Navon et~al.(2020)Navon, Shamsian, Chechik, and Fetaya]{navon2020learning}
Aviv Navon, Aviv Shamsian, Gal Chechik, and Ethan Fetaya.
\newblock Learning the pareto front with hypernetworks.
\newblock \emph{arXiv preprint arXiv:2010.04104}, 2020.

\bibitem[Ouyang et~al.(2022)Ouyang, Wu, Jiang, Almeida, Wainwright, Mishkin, Zhang, Agarwal, Slama, Ray, et~al.]{ouyang2022training}
Long Ouyang, Jeffrey Wu, Xu~Jiang, Diogo Almeida, Carroll Wainwright, Pamela Mishkin, Chong Zhang, Sandhini Agarwal, Katarina Slama, Alex Ray, et~al.
\newblock Training language models to follow instructions with human feedback.
\newblock \emph{Advances in Neural Information Processing Systems}, 35:\penalty0 27730--27744, 2022.

\bibitem[Peters and Schaal(2007)]{peters2007reinforcement}
Jan Peters and Stefan Schaal.
\newblock Reinforcement learning by reward-weighted regression for operational space control.
\newblock In \emph{Proceedings of the 24th international conference on Machine learning}, pages 745--750, 2007.

\bibitem[Rafailov et~al.(2023)Rafailov, Sharma, Mitchell, Ermon, Manning, and Finn]{rafailov2023direct}
Rafael Rafailov, Archit Sharma, Eric Mitchell, Stefano Ermon, Christopher~D Manning, and Chelsea Finn.
\newblock Direct preference optimization: Your language model is secretly a reward model.
\newblock \emph{arXiv preprint arXiv:2305.18290}, 2023.

\bibitem[Rame et~al.(2023)Rame, Couairon, Shukor, Dancette, Gaya, Soulier, and Cord]{rame2023rewarded}
Alexandre Rame, Guillaume Couairon, Mustafa Shukor, Corentin Dancette, Jean-Baptiste Gaya, Laure Soulier, and Matthieu Cord.
\newblock Rewarded soups: towards pareto-optimal alignment by interpolating weights fine-tuned on diverse rewards.
\newblock \emph{arXiv preprint arXiv:2306.04488}, 2023.

\bibitem[Roijers et~al.(2015)Roijers, Whiteson, and Oliehoek]{roijers2015computing}
Diederik~Marijn Roijers, Shimon Whiteson, and Frans~A Oliehoek.
\newblock Computing convex coverage sets for faster multi-objective coordination.
\newblock \emph{Journal of Artificial Intelligence Research}, 52:\penalty0 399--443, 2015.

\bibitem[Swamy et~al.(2024)Swamy, Dann, Kidambi, Wu, and Agarwal]{swamy2024minimaximalist}
Gokul Swamy, Christoph Dann, Rahul Kidambi, Zhiwei~Steven Wu, and Alekh Agarwal.
\newblock A minimaximalist approach to reinforcement learning from human feedback.
\newblock \emph{arXiv preprint arXiv:2401.04056}, 2024.

\bibitem[Taori et~al.(2023)Taori, Gulrajani, Zhang, Dubois, Li, Guestrin, Liang, and Hashimoto]{taori2023alpaca}
Rohan Taori, Ishaan Gulrajani, Tianyi Zhang, Yann Dubois, Xuechen Li, Carlos Guestrin, Percy Liang, and Tatsunori~B Hashimoto.
\newblock Alpaca: A strong, replicable instruction-following model.
\newblock \emph{Stanford Center for Research on Foundation Models. https://crfm. stanford. edu/2023/03/13/alpaca. html}, 3\penalty0 (6):\penalty0 7, 2023.

\bibitem[Touvron et~al.(2023{\natexlab{a}})Touvron, Lavril, Izacard, Martinet, Lachaux, Lacroix, Rozi{\`e}re, Goyal, Hambro, Azhar, et~al.]{llama1}
Hugo Touvron, Thibaut Lavril, Gautier Izacard, Xavier Martinet, Marie-Anne Lachaux, Timoth{\'e}e Lacroix, Baptiste Rozi{\`e}re, Naman Goyal, Eric Hambro, Faisal Azhar, et~al.
\newblock Llama: Open and efficient foundation language models.
\newblock \emph{arXiv preprint arXiv:2302.13971}, 2023{\natexlab{a}}.

\bibitem[Touvron et~al.(2023{\natexlab{b}})Touvron, Martin, Stone, Albert, Almahairi, Babaei, Bashlykov, Batra, Bhargava, Bhosale, et~al.]{llama2}
Hugo Touvron, Louis Martin, Kevin Stone, Peter Albert, Amjad Almahairi, Yasmine Babaei, Nikolay Bashlykov, Soumya Batra, Prajjwal Bhargava, Shruti Bhosale, et~al.
\newblock Llama 2: Open foundation and fine-tuned chat models.
\newblock \emph{arXiv preprint arXiv:2307.09288}, 2023{\natexlab{b}}.

\bibitem[Wang et~al.(2024)Wang, Lin, Xiong, Yang, Diao, Qiu, Zhao, and Zhang]{wang2024arithmetic}
Haoxiang Wang, Yong Lin, Wei Xiong, Rui Yang, Shizhe Diao, Shuang Qiu, Han Zhao, and Tong Zhang.
\newblock Arithmetic control of llms for diverse user preferences: Directional preference alignment with multi-objective rewards.
\newblock \emph{arXiv preprint arXiv:2402.18571}, 2024.

\bibitem[Yang et~al.(2024{\natexlab{a}})Yang, Liu, Xie, Zhang, Song, Huang, Kuang, and Ananiadou]{yang2024metaaligner}
Kailai Yang, Zhiwei Liu, Qianqian Xie, Tianlin Zhang, Nirui Song, Jimin Huang, Ziyan Kuang, and Sophia Ananiadou.
\newblock Metaaligner: Conditional weak-to-strong correction for generalizable multi-objective alignment of language models.
\newblock \emph{arXiv preprint arXiv:2403.17141}, 2024{\natexlab{a}}.

\bibitem[Yang et~al.(2024{\natexlab{b}})Yang, Pan, Luo, Qiu, Zhong, Yu, and Chen]{yang2024rewards}
Rui Yang, Xiaoman Pan, Feng Luo, Shuang Qiu, Han Zhong, Dong Yu, and Jianshu Chen.
\newblock Rewards-in-context: Multi-objective alignment of foundation models with dynamic preference adjustment.
\newblock \emph{arXiv preprint arXiv:2402.10207}, 2024{\natexlab{b}}.

\bibitem[Yang et~al.(2019)Yang, Sun, and Narasimhan]{yang2019generalized}
Runzhe Yang, Xingyuan Sun, and Karthik Narasimhan.
\newblock A generalized algorithm for multi-objective reinforcement learning and policy adaptation.
\newblock \emph{Advances in neural information processing systems}, 32, 2019.

\bibitem[Yuan et~al.(2023)Yuan, Yuan, Tan, Wang, Huang, and Huang]{yuan2023rrhf}
Zheng Yuan, Hongyi Yuan, Chuanqi Tan, Wei Wang, Songfang Huang, and Fei Huang.
\newblock Rrhf: Rank responses to align language models with human feedback without tears.
\newblock \emph{arXiv preprint arXiv:2304.05302}, 2023.

\bibitem[Zhang and Li(2007)]{zhang2007moea}
Qingfu Zhang and Hui Li.
\newblock Moea/d: A multiobjective evolutionary algorithm based on decomposition.
\newblock \emph{IEEE Transactions on evolutionary computation}, 11\penalty0 (6):\penalty0 712--731, 2007.

\bibitem[Zhang et~al.(2023{\natexlab{a}})Zhang, Chen, Bukharin, He, Cheng, Chen, and Zhao]{zhang2023adaptive}
Qingru Zhang, Minshuo Chen, Alexander Bukharin, Pengcheng He, Yu~Cheng, Weizhu Chen, and Tuo Zhao.
\newblock Adaptive budget allocation for parameter-efficient fine-tuning.
\newblock In \emph{The Eleventh International Conference on Learning Representations}, 2023{\natexlab{a}}.
\newblock URL \url{https://openreview.net/forum?id=lq62uWRJjiY}.

\bibitem[Zhang et~al.(2023{\natexlab{b}})Zhang, Lin, Xue, Chen, and Zhang]{zhang2023hypervolume}
Xiaoyuan Zhang, Xi~Lin, Bo~Xue, Yifan Chen, and Qingfu Zhang.
\newblock Hypervolume maximization: A geometric view of pareto set learning.
\newblock In \emph{Thirty-seventh Conference on Neural Information Processing Systems}, 2023{\natexlab{b}}.

\bibitem[Zhou et~al.(2011)Zhou, Qu, Li, Zhao, Suganthan, and Zhang]{zhou2011multiobjective}
Aimin Zhou, Bo-Yang Qu, Hui Li, Shi-Zheng Zhao, Ponnuthurai~Nagaratnam Suganthan, and Qingfu Zhang.
\newblock Multiobjective evolutionary algorithms: A survey of the state of the art.
\newblock \emph{Swarm and evolutionary computation}, 1\penalty0 (1):\penalty0 32--49, 2011.

\bibitem[Zhou et~al.(2023)Zhou, Liu, Yang, Shao, Liu, Yue, Ouyang, and Qiao]{zhou2023beyond}
Zhanhui Zhou, Jie Liu, Chao Yang, Jing Shao, Yu~Liu, Xiangyu Yue, Wanli Ouyang, and Yu~Qiao.
\newblock Beyond one-preference-for-all: Multi-objective direct preference optimization.
\newblock \emph{arXiv preprint arXiv:2310.03708}, 2023.

\bibitem[Zhu et~al.(2023)Zhu, Wu, Hu, Yan, Hsieh, Hou, and Wu]{zhu2023sample}
Yiheng Zhu, Jialu Wu, Chaowen Hu, Jiahuan Yan, Chang-Yu Hsieh, Tingjun Hou, and Jian Wu.
\newblock Sample-efficient multi-objective molecular optimization with gflownets.
\newblock \emph{arXiv preprint arXiv:2302.04040}, 2023.

\end{thebibliography}

\newpage


\appendix
\onecolumn

\clearpage





\doparttoc 
\faketableofcontents 
\renewcommand \thepart{}
\renewcommand \partname{}

\addcontentsline{toc}{section}{Appendix} 
\part{Supplementary Material} 
\parttoc 
\newpage
\section{Preliminary Theoretical Results}
\label{app:pre-thry}
In this section, we prove the validity of combining reward models of all preference dimensions through linear scalarization in the RLHF optimization procedure, even though each reward model solved by the Bradley-Terry (BT) model \cite{bradley1952rank} is not uniquely determined. This is formalized in the following lemma.

\begin{lemma} [Extension of Lemma 2 in \cite{rafailov2023direct} for multiple reward models]\label{lem:pre}
    Let \( r_i(x,y) \) and \( r_i'(x,y) \) be equivalent reward models for the $i$-th preference dimension, where \( r_i'(x,y) = r_i(x,y) + \phi_i(x) \). The linear combinations \( r(x,y) = \sum_{i=1}^m \lambda_i r_i(x,y) \) and \( r'(x,y) = \sum_{i=1}^m \lambda_i r_i(x,y) + \sum_{i=1}^m \lambda_i \phi_i(x) \) induce the same optimal policy in the constrained RL problem $\max_{\pi} J_{\text{RLHF}}(\pi) = \mathbb{E}_{x\sim\mathcal{D}}\left[\mathbb{E}_{y \sim \pi(\cdot|x)}\left[r(x, y)\right] - \beta\mathbb{D}_{\text{KL}}\left[\pi(\cdot|x)||\pi_{\text{ref}}(\cdot|x)\right]\right]$, where $\beta$ is a positive punishment factor of the KL constraint. 
\end{lemma}

\begin{remark}
This lemma demonstrates that it is valid to linearly combine reward models of all dimensions, even if the reward models are not uniquely identified. It is used in analyzing the limitations of single-objective alignment and it validates the LS aggregation employed with Panacea.
\end{remark}

Below, we provide a concise proof of \Cref{lem:pre}.
\begin{proof}
According to the constrained RL literatures \cite{peters2007reinforcement,berry2011permutation}, the policy for the reward function \(r'(x,y)\) in a Kullback-Leibler (KL) constrained reinforcement learning (RL) problem can be formulated as follows:

$$
\pi_{r'}(y | x) = \frac{\pi_{\text{ref}}(y | x) \exp \left(\frac{1}{\beta} r'(x, y)\right)}{\sum_y \pi_{\text{ref}}(y | x) \exp \left(\frac{1}{\beta} r'(x, y) \right)}.
$$

Expanding the term in \(r'(x,y)\), we obtain:

$$
\pi_{r'}(y | x) = \frac{\pi_{\text{ref}}(y | x) \exp \left( \frac{1}{\beta}\left( \sum_{i=1}^m \lambda_i r_i(x,y) + \underbrace{\sum_{i=1}^m \lambda_i \phi_i(x)}_{\phi^\prime(x)}  \right)\right)}{\sum_y \pi_{\text{ref}}(y | x) \exp \left(\frac{1}{\beta}\left( \sum_{i=1}^m \lambda_i r_i(x,y) + \underbrace{\sum_{i=1}^m \lambda_i \phi_i(x)}_{\phi^\prime(x)} \right)\right)}.
$$

Upon simplifying by canceling out the common term $\exp(\phi^\prime(x))$, we get:
\[
\pi_{r'}(y | x) = \frac{\pi_{\text{ref}}(y | x) \exp \left(\frac{1}{\beta} r(x, y)\right) \cancel{ \exp \left( \frac{1}{\beta} (\phi^\prime(x) ) \right) } }{\sum_y \pi_{\text{ref}}(y | x) \exp \left(\frac{1}{\beta} r(x, y)\right) \cancel{\exp \left(\frac{1}{\beta} ( \phi^\prime(x) )\right)}} = \pi_r (y | x),
\]
which completes the proof. 



\end{proof}

\section{The Limitation of Single-Objective Alignment}
\label{app:single-analysis}
In the following content, we provide a theoretical analysis that the model trained by the single-objective alignment paradigm could actually misalign with every labeler. We conduct analysis on RLHF, the most common approach. We make the following assumptions:

\begin{ass}\label{ass:BT}
    Human preference can be modeled by the Bradley-Terry model \cite{bradley1952rank}.
\end{ass}

\begin{ass}
    Different people are consistent in labeling each preference dimension.
\end{ass}

These two assumptions imply that people possess the same reward model $r_i(x, y)$ for each preference dimension $i$.

\begin{ass}
    The synthesized reward model of a person is the LS of per-dimensional reward models according to his/her preference vector under a shift invariant term (c.f \cite{rafailov2023direct}[Lemma1]). That is,
    \begin{equation}
        r(x,y) = \sum_{i = 1}^m \lambda_i r_i(x,y) + \phi(x).
    \end{equation}
\end{ass}

Now we prove the main theoretical result.
\begin{theorem}
    Consider the case where there are $n$ labelers in total. Each labeler $h$ labels a portion $p^h$ of the entire dataset, where $p^h \in [0,1], \sum_{h=1}^n p^h = 1$. The preference vector of labeler $h$ is $\vlam^h = (\lambda^h_1, \lambda^h_2, \ldots, \lambda^h_m)$. The labelers have different preference vectors, \emph{i.e.} $\exists\ j, h \in \{1, \ldots, n\}, \vlam^j\neq\vlam^h$. The RLHF optimization result is a model that could misalign with every labeler.
\end{theorem}

\begin{proof}
    The reward model $r^h$ of labeler $h$ is $r^h(x,y) = \sum_{i=1}^m\lambda^h_ir_i(x, y) + \phi^h(x)$. $J^h(\theta)$ denotes the optimization objective corresponding to the reward model of labeler $h$. The joint optimization objective is

\begin{align}
    &\max_\theta\ \sum_{h=1}^np^hJ^h( \pi_\theta ) \nonumber \\
    & \text{(Substituting the oracle reward function.)} \\
    = & \max_\theta \sum_{h=1}^np^h\left(\mathbb{E}_{x \sim \mathcal{D}}\left[\mathbb{E}_{y \sim \pi_{\theta}(\cdot|x)}\left[r^h(x, y)\right] - \beta\mathbb{D}_{\text{KL}}\left[\pi_{\theta}(\cdot|x)||\pi_{\text{ref}}(\cdot|x)\right]\right]\right) \nonumber \\
     & \text{(Rearrange reward terms.)} \\
     = & \max_\theta \mathbb{E}_{x \sim \mathcal{D}}\left[\mathbb{E}_{ y \sim \pi_\theta(\cdot|x)} \left[\sum_{h=1}^n p^h r^h(x, y) \right] 
     -\beta\mathbb{D}_{\text{KL}} \left[ \pi_\theta(\cdot|x)||\pi_{\text{ref}}(\cdot|x) \right]\right] \nonumber \\
     = & \max_\theta \mathbb{E}_{x \sim \mathcal{D}}\left[\mathbb{E}_{ y \sim \pi_\theta(\cdot|x)} \left[\sum_{h=1}^n p^h \left(\sum_{i=1}^m\lambda^h_ir_i(x, y) + \phi^h(x)\right) \right] 
     -\beta\mathbb{D}_{\text{KL}} \left[ \pi_\theta(\cdot|x)||\pi_{\text{ref}}(\cdot|x) \right]\right] \nonumber \\
     & \text{(Define $\varphi(x) \defeq \sum_{h=1}^n p^h \phi^h(x)$)} \\
     = & \max_\theta \mathbb{E}_{x \sim \mathcal{D}}\left[\mathbb{E}_{ y \sim \pi_\theta(\cdot|x)} \left[\sum_{h=1}^n \sum_{i=1}^m p^h\lambda^h_ir_i(x, y) + \varphi(x) \right] 
     -\beta\mathbb{D}_{\text{KL}} \left[ \pi_\theta(\cdot|x)||\pi_{\text{ref}}(\cdot|x) \right] \right]\nonumber \\
     = & \max_\theta \mathbb{E}_{x \sim \mathcal{D}}\left[\mathbb{E}_{ y \sim \pi_\theta(\cdot|x)} \left[ \sum_{i=1}^m \sum_{h=1}^n p^h\lambda^h_ir_i(x, y) + \varphi(x) \right] 
     -\beta\mathbb{D}_{\text{KL}} \left[ \pi_\theta(\cdot|x)||\pi_{\text{ref}}(\cdot|x) \right] \right]\nonumber \\
     = & \max_\theta \mathbb{E}_{x \sim \mathcal{D}}\left[\mathbb{E}_{ y \sim \pi_\theta(\cdot|x)} \left[ \sum_{i=1}^m \left(\sum_{h=1}^n p^h\lambda^h_i\right)r_i(x, y) + \varphi(x) \right] 
     -\beta\mathbb{D}_{\text{KL}} \left[ \pi_\theta(\cdot|x)||\pi_{\text{ref}}(\cdot|x) \right] \right]\nonumber \\
     & \text{(Define $\lambda_i^{\text{opt}} \defeq \sum_{h=1}^np^h\lambda^h_i, i = 1, \ldots, m$)} \\
     = & \max_\theta \mathbb{E}_{x \sim \mathcal{D}}\left[\mathbb{E}_{ y \sim \pi_\theta(\cdot|x)} \left[ \sum_{i=1}^m \lambda^{\text{opt}}_ir_i(x, y) + \varphi(x) \right] 
     -\beta\mathbb{D}_{\text{KL}} \left[ \pi_\theta(\cdot|x)||\pi_{\text{ref}}(\cdot|x) \right] \right]\nonumber \\
 \end{align}

Thus, we show that it actually optimizes with the preference vector $\vlam^{\text{opt}}$, with $\lambda_i^{\text{opt}} = \sum_{h=1}^np^h\lambda^h_i, i = 1, \ldots, m$. According to the constrained RL literatures \cite{peters2007reinforcement,berry2011permutation}, the corresponding optimal policy can be expressed as:
\begin{equation}
    \pi_\theta^*(y | x)=\frac{1}{Z(x)} \pi_{\text {ref }}(y | x) \exp \left(\frac{1}{\beta} \sum_{i=1}^m \lambda_i^{\text{opt}} r_i(x, y) \right).
\end{equation}

It is important to note that this optimal preference vector may not align with the individual preferences of each annotator. As a result, the trained model may not fully reflect the labeling criteria of any single annotator, potentially leading to discrepancies in the model's predictions.

\end{proof}

\section{Theoretical Support for Panacea with LS / Tche function} \label{sec:represent}
In the following content, we prove for \Cref{thm:full:represent} from the main paper, showing that both linear and Tchebycheff scalarization can recover the entire Pareto Front (PF) under practical assumptions. The proof has two subsections: first for the linear scalarization function in \Cref{sec:proof:c1}, followed by the Tchebycheff aggregation function in \Cref{sec:proof:c2}.

\subsection{Proof for LS Aggregation Function} \label{sec:proof:c1}
We provide a proof sketch for this part.

\textbf{Step 1:} Under the full categorical representation assumption, for any two policies \( \pi^{(a)}(\cdot | x) \) and \( \pi^{(b)}(\cdot | x) \), we can create a new policy (\(\pi'\)) that, with probability (w.p.) \( p \) (where \( 0 \leq p \leq 1 \)), takes \( \pi^{(a)}(\cdot | x) \) and w.p. \( 1-p \), takes \( \pi^{(b)}(\cdot | x) \). This policy can also be represented by LLM.

\textbf{Step 2:} Using the above policy construction method, we prove that the objective spaces of DPO, RLHF, and SFT are convex.

\textbf{Step 3:} When the objective spaces are convex, the Pareto objectives found by LS aggregation function (Convex coverage set (CCS)) equal the entire Pareto front.

\textbf{Step 4:} By optimizing the Panacea objective function \(\E_{\vlam \in \Delta_m} \left[g^\mathrm{LS}_\vlam(\theta)\right]\), we can recover the entire Pareto front.

Then, we start our formal proof. We first restate the assumption for the full categorical policy space in \Cref{thm:full:represent}. 
\begin{ass}[Full Categorical Policy Space Assumption (detailed restatement from Assumption 2 in \Cref{thm:full:represent})] \label{ass:app:full}
For a specific preference vector $\vlam$, the LLM policy space formed by all $y \sim \pi_{\theta, \vlam}(\cdot | x)$ can represent all the categorical distribution set $\Pi(x)$ for response $y = [t_1, \ldots, t_N]$, where $N$ is the response length and $t_i$ denote each token, given an input sentence $x$.
\end{ass}
This assumption is proper because the probability of each token \( t_1, \ldots, t_N \) ($N$ denotes the length of the output of $y$) can be represented by a LLM policy. Given the strong representation ability of LLMs, any probability value of token sequence \( t_1, \ldots, t_N \) can be represented by their output. With this assumption, a direct corollary holds because the linear combination of categorical distributions is still a categorical distribution.  

As a corollary of \Cref{ass:app:full}, we have: 
\begin{corollary} \label{coro:inter}
    For two policies \( \pi^{(a)}(\cdot | x) \) and \( \pi^{(b)}(\cdot | x) \), a new policy \( \pi^\prime \) w.p. \( p \) (\( 0 \leq p \leq 1 \)) follows \( \pi^{(a)}(\cdot | x) \) and w.p. \( 1-p \) follows $\pi^{(b)}(\cdot | x)$ belongs to the categorical distribution \( \Pi(x) \).
\end{corollary}
The reason for that is such constructed policy is still a categorical distribution. For the next step, we use this corollary to prove the following lemma to show that the objective spaces \(\vJ_\mathrm{SFT}\), \(\vJ_\mathrm{RLHF}\), and \(\vJ_\mathrm{DPO}\) are convex.

\begin{lemma}[ Convex space Lemma, adapted from \cite{hu2024revisiting}(Eq. 13) ]
\label{lem:convex}
    For any two objectives \(\vJ^{(a)}_\mathrm{alg}\) and \(\vJ^{(b)}_\mathrm{alg}\), and for any \(0 < \alpha < 1\), there exists a policy \(\pi' \in \Pi(x)\) such that \(\alpha \vJ^{(a)}_\mathrm{alg} + (1-\alpha) \vJ^{(b)}_\mathrm{alg} = \vJ(\pi')\), where $\vJ_\mathrm{alg}$ can be $\vJ_\mathrm{DPO}$, $\vJ_\mathrm{SFT}$, or $\vJ_\mathrm{RLHF}$.    
\end{lemma}

This lemma mainly follows from Eq. 13 in \cite{hu2024revisiting}. We include their proof for our purpose for completeness. The objectives \(\vJ_\mathrm{SFT}\), \(\vJ_\mathrm{RLHF}\), and \(\vJ_\mathrm{DPO}\) can all be written as \(\vJ_\mathrm{alg}(\pi) = \E_{{x,y} \in D} [\vf(x,y, \pi(y | x))]\) for some particular design of $\vf(x,y, \pi(y | x))$. For any \(0 \leq \alpha \leq 1\), by \Cref{coro:inter}, we can construct a new policy \(\pi'\) and a uniform random variable \(S \sim U(0,1)\) such that:
\[
\pi'(y|x) = 
\begin{cases} 
    \pi^a(y|x) & \text{if } S < \alpha \\
    \pi^b(y|x) & \text{if } S \geq \alpha 
\end{cases}
\]
Then,
\begin{align*}
    \vJ(\pi') &= \E_{(x,y) \sim \mathcal{D}} [\vf(x,y, \pi'(y | x))] \\ 
    &= \E_{S \sim U(0,1)} \E_{(x,y) \sim \mathcal{D}} [\vf(x,y, \pi'(y | x)) | S] \\ 
    &= \alpha \E_{(x,y) \sim \mathcal{D}} [\vf(x,y, \pi'(y | x)) | S < \alpha] + (1-\alpha) \E_{(x,y) \sim \mathcal{D}} [\vf(x,y, \pi'(y | x)) | S \geq \alpha] \\ 
    &= \alpha \E_{(x,y) \sim \mathcal{D}} [\vf(x,y, \pi^{(a)}(y | x))] + (1-\alpha) \E_{(x,y) \sim \mathcal{D}} [\vf(x,y, \pi^{(b)}(y | x))] \\ 
    &= \alpha \vJ(\pi^{(a)}) + (1-\alpha) \vJ(\pi^{(b)})
\end{align*}

Thus, for any convex combination of \(\vJ(\pi^{(a)})\) and \(\vJ(\pi^{(b)})\), there exists a policy \(\pi'\) such that \(\vJ(\pi') = \alpha \vJ(\pi^{(a)}) + (1-\alpha) \vJ(\pi^{(b)})\), indicating that the space of \(\vJ(\pi)\) is convex. We denote the full space of \(\vJ(\pi)\) for all policies as \(\sJ\).

For the third step, we use \Cref{lem:convex} to establish that linear scalarization functions have the capability to discover the complete PF by traversing the entire preference simplex $\Delta_m$ (i.e., the approach employed in Panacea). To prove for that, we introduce the concept of the convex coverage set, which is the objective set that can be found by optimizing the linear scalization function with all preference vector $\boldsymbol{\lambda} \in \Delta_m$. 
We now define CCS, which is the set of solutions can be found LS. 
\begin{definition}[Convex Coverage Set (CCS), adapted from \cite{roijers2015computing}(Def. 9)]
   The CCS contains the objective such that there exists a preference vector \(\vlam\) where the inner product of \(\vlam\) and this objective is greater than that of \(\vlam\) with any other objective vectors in the objective space.
    CCS $:=\left\{\vJ \in \sJ | \exists \vlam \in \Delta_m \right.$ s.t. $\vlam^{\top} \vJ \geq \vlam^{\top} \vJ', \forall \vJ' \in \sJ \}$.
\end{definition}
Finally, we prove for that that when the objective space is convex, the linear scalarization can recover the whole Pareto objective set, i.e., $\pf = \ccs$, where $\pf$ denote the objective vectors forming the Pareto front.  
\begin{proof}
    The PF $\pf$ is a subset of the boundary of the objective space, denoted as $\partial(\vJ(\Pi))$. By proving that $\vJ(\Pi)$ is a convex set, we can apply the supporting hyperplane theorem \cite{boyd2004convex} (Sec. 2.5.2). According to this theorem, for every element $\vr$ in $\partial(\vJ(\Pi))$, there exists $\vlam \in \mathbb{R}$ such that $\vlam^T (\vr - \vr') \geq 0$ for all $\vr' \in \vJ(\Pi)$. Moreover, when $\vr$ is Pareto optimal, such $\vlam \succeq 0$. Hence, we have $\vlam^T (\vr - \vr') \geq 0$ for all $\vr' \in \vJ(\Pi)$ and $\vlam \in \Delta_m$. This condition implies that $\pf \subset \ccs$. Since it has been established that $\ccs \subset \pf$, we can conclude that $\ccs = \pf$.
\end{proof}

For the last step, we demonstrate that by optimizing \(\E_{\vlam \in \Delta_m} \left[g^\mathrm{LS}_\vlam(\theta)\right]\) using the LS aggregation function, we can recover almost the entire Pareto front. This is because, if a larger non-zero measure Pareto front could not be found, it implies that there exist non-zero measure preference vectors that would make the expectation function value \(\E_{\vlam \in \Delta_m} \left[g^\mathrm{LS}_\vlam(\theta)\right]\) exceed its optimal value, which is contradictory of our assumption.

\subsection{Proof for Tchebycheff Aggregation Function} \label{sec:proof:c2}
To prove that using the Tchebycheff aggregation function allows Panacea to recover the full Pareto front, we introduce the following lemma:

\begin{lemma}[Adapted from \cite{choo1983proper}, Theorem 3.1] \label{lemma:tche}
A feasible solution $\theta$ is Pareto optimal if and only if there exists a weight vector $\lambda$ such that $\theta$ is an optimal solution to the aggregation function (\Cref{eqn:tche}) defined in the main paper.
\end{lemma}

Using this lemma and assuming Panacea can represent the Pareto policy under all preferences (Assumption 1 in \Cref{thm:full:represent}), optimizing the expectation loss

\[
-\E_{\vlam \in \Delta_m} g^\mathrm{Tche}_\vlam(\theta)
\]

allows Panacea to recover almost every policy.

\begin{proof}
    If a non-Pareto policy has a measure greater than zero, then according to \Cref{lemma:tche}, there exists a preference set of greater than zero measure where the non-Pareto policy has a higher value compared to the optimal value of the Tchebycheff function under the corresponding preferences. This implies that \(\E_{\vlam \in \Delta_m} g^\mathrm{Tche}_\vlam(\theta)\) has not been optimized to its optimal value, contradicting Assumption 1 in \Cref{thm:full:represent}.
\end{proof}

\section{Aggregated Training Objectives for Panacea}
\label{app:agg}
In this section, we present the LS / Tche aggregated training objectives for Panacea with RLHF / DPO / SFT. In RLHF, reward models $r_i, i = 1, \ldots, m$ are learned for each preference dimension. For a specific preference vector, the LS aggregated objective function is

\begin{align}
\max_{\theta} g^\mathrm{LS}_\vlam(\theta) =\max_{\theta}\ \mathbb{E}_{x \sim \mathcal{D}}\left[\mathbb{E}_{y \sim \pi_{\theta, \vlam}(\cdot|x)}\left[\sum_{i = 1}^m\lambda_ir_i(x, y)\right] - \beta\mathbb{D}_{\text{KL}}\left[\pi_{\theta, \vlam}(\cdot|x)||\pi_{\text{ref}}(\cdot|x)\right]\right].    
\end{align}

The Tche aggregated objective is

\begin{align}
    \max_{\theta} g^\mathrm{Tche}_\vlam(\theta) =\max_{\theta}\ \mathbb{E}_{x \sim \mathcal{D}}\left[\mathbb{E}_{y \sim \pi_{\theta, \vlam}(\cdot|x)}\left[-\max_{1\leq i\leq m}\lambda_i (z_i - r_i(x, y))\right] -\beta\mathbb{D}_{\text{KL}}\left[\pi_{\theta,\vlam}(\cdot|x)||\pi_{\text{ref}}(\cdot|x)\right]\right],
\end{align}

where $z_i$ is the maximum reward for preference dimension $i$. Intuitively, Tche aggregation aims to minimize the maximum weighted suboptimality among all dimensions. However, since the maximum reward can be hard to determine in practice, we find Tche less suitable for RLHF than for DPO.

DPO transforms the reinforcement learning objective into a supervised objective, whose LS aggregated objective is
\begin{align}\label{eq:dpo-ls}
\max_{\theta} g^\mathrm{LS}_\vlam(\theta) =& \max_{\theta} \sum_{i=1}^m \lambda_iJ_{\text{DPO}, i}(\pi_{\theta, \vlam}) \nonumber \\
    =& \max_{\theta}\sum_{i=1}^m \lambda_i\E_{(x, y_w, y_l)\sim\mathcal{D}_i} \left[\log\sigma\left(\beta \log \frac{\pi_{\theta, \vlam}\left(y_w | x\right)}{\pi_{\mathrm{ref}}\left(y_w | x\right)}-\beta \log \frac{\pi_{\theta, \vlam}\left(y_l | x\right)}{\pi_{\mathrm{ref}}\left(y_l | x\right)}\right)\right].
\end{align}

To derive the Tche aggregated objective, we have
\begin{align}\label{eq:dpo-tche}
\max_{\theta} g^\mathrm{Tche}_\vlam(\theta) = & \max_{\theta} \min_{1\leq i \leq m} \lambda_i (J_{\text{DPO}, i}(\pi_{\theta, \vlam}) - z_i) \nonumber \\
= & \max_{\theta} \min_{1\leq i \leq m} \lambda_i J_{\text{DPO}, i}(\pi_{\theta, \vlam}) \nonumber \\
= & \max_{\theta} \min_{1\leq i \leq m} \lambda_i \E_{(x, y_w, y_l)\sim\mathcal{D}_i} \left[\log\sigma\left(\beta \log \frac{\pi_{\theta, \vlam}\left(y_w | x\right)}{\pi_{\mathrm{ref}}\left(y_w | x\right)}-\beta \log \frac{\pi_{\theta, \vlam}\left(y_l | x\right)}{\pi_{\mathrm{ref}}\left(y_l | x\right)}\right)\right]
\end{align}

Since the optimal value $z_i$ for per-dimension DPO objective is $0$, this is naturally compatible with Tche aggregation.

Finally, the LS aggregated SFT objective is

\begin{align}\label{eq:sft-ls}
\max_{\theta} g^\mathrm{LS}_\vlam(\theta) = \max_{\theta} \sum_{i=1}^m \lambda_iJ_{\text{SFT}, i}(\pi_{\theta, \vlam})
    = \max_{\theta}\sum_{i=1}^m \lambda_i\mathbb{E}_{(x, y)\sim\mathcal{D}_i}\left[\log \pi_{\theta, \vlam}(y|x)\right].
\end{align}
Similar to DPO, since the optimal value $z_i$ for per-dimension SFT objective is 0, the Tche aggregation of SFT objectives is
\begin{align}\label{eq:sft-tche}
\max_{\theta} g^\mathrm{Tche}_\vlam(\theta) = & \max_{\theta} \min_{1\leq i \leq m} \lambda_i (J_{\text{SFT}, i}(\pi_{\theta, \vlam}) - z_i) \nonumber \\
= & \max_{\theta} \min_{1\leq i \leq m} \lambda_i J_{\text{SFT}, i}(\pi_{\theta, \vlam}) \nonumber \\
= & \max_{\theta} \min_{1\leq i \leq m} \lambda_i \mathbb{E}_{(x, y)\sim\mathcal{D}_i}\left[\log \pi_{\theta, \vlam}(y|x)\right].
\end{align}

\section{Experiment Details and Additional Results}
\label{sec:app:details}


In this section, we present experimental details including computational resources, algorithm implementation, data curation, experiment setup, and evaluation details, and analyze additional results. All our experiments are conducted on an 8$\times$A800-80GB GPU server. Other details are elaborated below.

\subsection{Core Implementation of Panacea}
\label{app:code}

\begin{figure}[h]
    \centering
    \includegraphics[width=1\columnwidth]{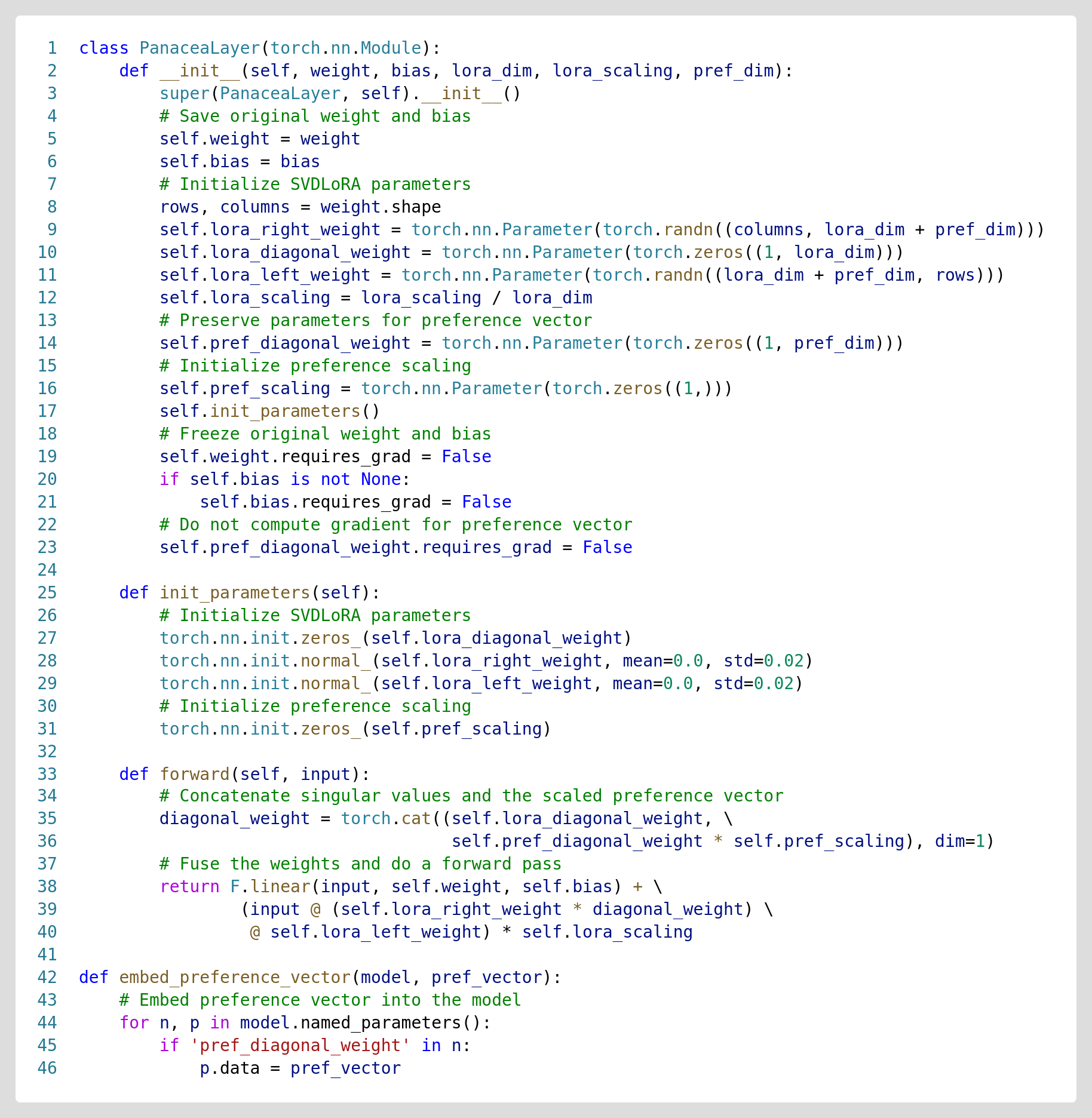}
    \vspace{-10pt}
    \caption{Core implementation of Panacea.}
    \vspace{-10pt}
\label{fig:panacea-code}
\end{figure}

Our implementation is based on the Safe-RLHF \cite{dai2023safe} codebase. As described in \Cref{sec:panacea} and visualized in \Cref{fig:Panacea}, the core design of Panacea is the embedding of the preference vector as singular values based on SVD-LoRA. Its core code is presented in \Cref{fig:panacea-code}. In our experiments, we perform Panacea adaptation to all self-attention and MLP layers. We initialize the singular values and preference scaling to zero, so as not to impact the model behavior at the beginning of training \cite{hu2022lora, zhang2023adaptive}. In each iteration, we sample a preference vector from the preference simplex, embed it into the model, and train the model on the aggregated objective.

\subsection{Data Curation}
\label{app:data-curation}

In the helpful-harmless (HH) problem in \Cref{sec:HH}, we use the BeaverTails dataset \cite{ji2023beavertails}, which contains both helpfulness and harmlessness preference labels. In the augmented helpful-harmless-concise (HHC) problem in \Cref{sec:HHC}, we again use the BeaverTails dataset. For RLHF, we define the reward model as a rectified affine function,

$$
r_{\text{concise}}(x, y) = \begin{cases}
    r_{\text{max}},\ l_y \leq c  \\
    r_{\text{max}} + 1 - \frac{l_y}{c}, \ \text{otherwise} \\
\end{cases}
$$
where $r_{\text{max}}$ defines the maximum reward, $l_y$ denotes token length of response $y$, and $c$ defines both the threshold for maximum reward and the slope of concise reward model. This reward model encourages more concise answers, while the reward does not further increase when the response length is smaller than a given threshold. For DPO, we label the shorter response to each prompt as preferred. 

In the Chat multi-dimensional alignment problem in \Cref{sec:chat-multi-dim}, we curate SFT data by letting Llama-3-8B-Instruct \cite{llama3modelcard} generate responses for Alpaca prompts \cite{taori2023alpaca} in each dimension. Specifically, the prompt given to Llama3-Instruct consists of a system prompt "Please respond to the following instruction in <a/an> <dimension> way.", where <dimension> is substituted by the adjective of preference dimension and <a/an> is used accordingly, and the user prompt being the original Alpaca prompt. We employ vLLM \cite{kwon2023efficient} for fast model inference to accelerate data generation.

\subsection{Experiment Setup}
\label{app:exp-config}


\begin{table}[h]
\centering
\caption{Common hyperparams of Panacea with RLHF.}
\renewcommand{\arraystretch}{1.25}
\begin{tabular}{ll|ll}
\hlineB{3}
Hyperparams             & Values    & Hyperparams                & Values      \\ \hlineB{2}
max\_length                    & 512       & critic\_weight\_decay           & 0.0         \\
kl\_coeff                      & 0.02      & critic\_lr\_scheduler\_type     & ``constant" \\
clip\_range\_ratio             & 0.2       & critic\_lr\_warmup\_ratio       & 0.03        \\
clip\_range\_score             & 50.0      & critic\_gradient\_checkpointing & true        \\
clip\_range\_value             & 5.0       & normalize\_reward               & false       \\
epochs                         & 2         & seed                            & 42          \\
update\_iters                  & 1         & fp16                            & false       \\
gradient\_accumulation\_steps  & 2         & bf16                            & true        \\
actor\_lr                      & 0.002     & tf32                            & true        \\
actor\_weight\_decay           & 0.01      & lora\_dim                       & 8           \\
actor\_lr\_scheduler\_type     & ``cosine" & lora\_scaling                   & 512         \\
actor\_lr\_warmup\_ratio       & 0.03      & only\_optimize\_lora            & true        \\
actor\_gradient\_checkpointing & true      & lora\_module\_name              & ``layers."  \\
critic\_lr                     & 0.001     & num\_return\_sequences          & 1           \\
repetition\_penalty            & 1.0       & temperature                     & 1.0         \\
top\_p                         & 1.0       &                                 &             \\ \hlineB{3}
\end{tabular}
\label{tab:Panacea-rlhf-common-config}
\end{table}

In this part, we present details about the experiment setup. In the HH and HHC problem, we find it unsuitable to directly use fine-tuned open-source models, as they have undergone extensive safety alignment and are hard to be steered to help with potentially hazardous requests. Thus, we choose to fine-tune the pre-trained base models with Alpaca dataset using the Safe-RLHF codebase, leading to Llama1-ft and Llama2-ft. The reward models are trained upon these SFT models. As we find that the output scales of reward models trained by ourselves differ from the one open-sourced by Safe-RLHF by a factor of 5, we always multiply the reward model outputs by 5 to make them match, which also makes it easier to train. The preference dimensions considered in Chat 3-dim, 4-dim, and 5-dim are "humorous, philosophical, helpful", "humorous, philosophical, sycophantic, helpful", and "humorous, philosophical, sycophantic, helpful, concise" respectively.
As for the rank of Panacea, we always fix $k$ to 8, and $m$ equals the number of preference dimensions. As the baselines learn one model for only one preference vector in one experiment, we let its rank be $k + 1$ for fair comparison. When sampling from the preference simplex, we sample the vertices, \emph{i.e.} $(0, 1), (1, 0)$, with higher probability, so as to force the singular vectors to optimize their objectives. In \Cref{tab:Panacea-rlhf-common-config}, \Cref{tab:Panacea-dpo-common-config}, and \Cref{tab:Panacea-sft-common-config} we provide the common hyperparameters for Panacea with RLHF, DPO, and SFT. Different hyperparameters include: in HH with RLHF and Llama1-ft, \texttt{batch\_size} $=16$, \texttt{ptx\_coeff} $=16$; in HH and HHC with RLHF and Llama2-ft, \texttt{batch\_size} $=8$, \texttt{ptx\_coeff} $=4$; in HH with DPO and Llama1-ft, \texttt{learning\_rate} $=0.0002$; in HH and HHC with DPO and Llama2-ft, \texttt{learning\_rate} $=0.001$; in Chat 3, 4, 5-dim with SFT and Llama3-Instruct, \texttt{batch\_size} $=16$; in Chat 10-dim with SFT and Llama3-Instruct, \texttt{batch\_size} $=8$. We also note that in HHC with RLHF experiment, the concise reward model is defined with \texttt{max\_concise\_reward} $=4$ and \texttt{concise\_scale}$=50$. RS is trained with the same hyperparameters.



\begin{table}[t]
\centering
\caption{Common hyperparams of Panacea with DPO.}
\renewcommand{\arraystretch}{1.25}
\begin{tabular}{ll|ll|ll}
\hlineB{3}
Hyperparams              & Values & Hyperparams     & Values     & Hyperparams & Values \\ \hlineB{2}
max\_length                   & 512    & lora\_dim            & 8          & epochs           & 1      \\
scale\_coeff                  & 0.1    & lora\_scaling        & 512        & seed             & 42     \\
weight\_decay                 & 0.05   & only\_optimize\_lora & true       & fp16             & false  \\
batch\_size                   & 16     & lora\_module\_name   & ``layers." & bf16             & true   \\
gradient\_checkpointing       & true   & lr\_warmup\_ratio    & 0.03       & tf32             & true   \\
gradient\_steps & 1      & lr\_scheduler\_type  & ``cosine"  &                  &        \\ \hlineB{3}
\end{tabular}
\label{tab:Panacea-dpo-common-config}
\end{table}


\begin{table}[t]
\centering
\caption{Common hyperparams of Panacea with SFT.}
\renewcommand{\arraystretch}{1.25}
\begin{tabular}{ll|ll|ll}
\hlineB{3}
Hyperparams              & Values    & Hyperparams     & Values     & Hyperparams & Values \\ \hlineB{2}
max\_length                   & 512       & lora\_dim            & 8          & epochs           & 4      \\
weight\_decay                 & 0.0       & lora\_scaling        & 512        & seed             & 42     \\
learning\_rate                & 0.0002    & only\_optimize\_lora & true       & fp16             & false  \\
gradient\_checkpointing       & true      & lora\_module\_name   & ``layers." & bf16             & true   \\
gradient\_steps & 2         & lr\_warmup\_ratio    & 0.03       & tf32             & true   \\
lr\_scheduler\_type           & ``cosine" &                      &            &                  &        \\ \hlineB{3}
\end{tabular}
\label{tab:Panacea-sft-common-config}
\end{table}



\begin{wrapfigure}{r}{0.39\textwidth}
\vspace{-10pt}
  \begin{center}
    \includegraphics[width=0.31\textwidth]{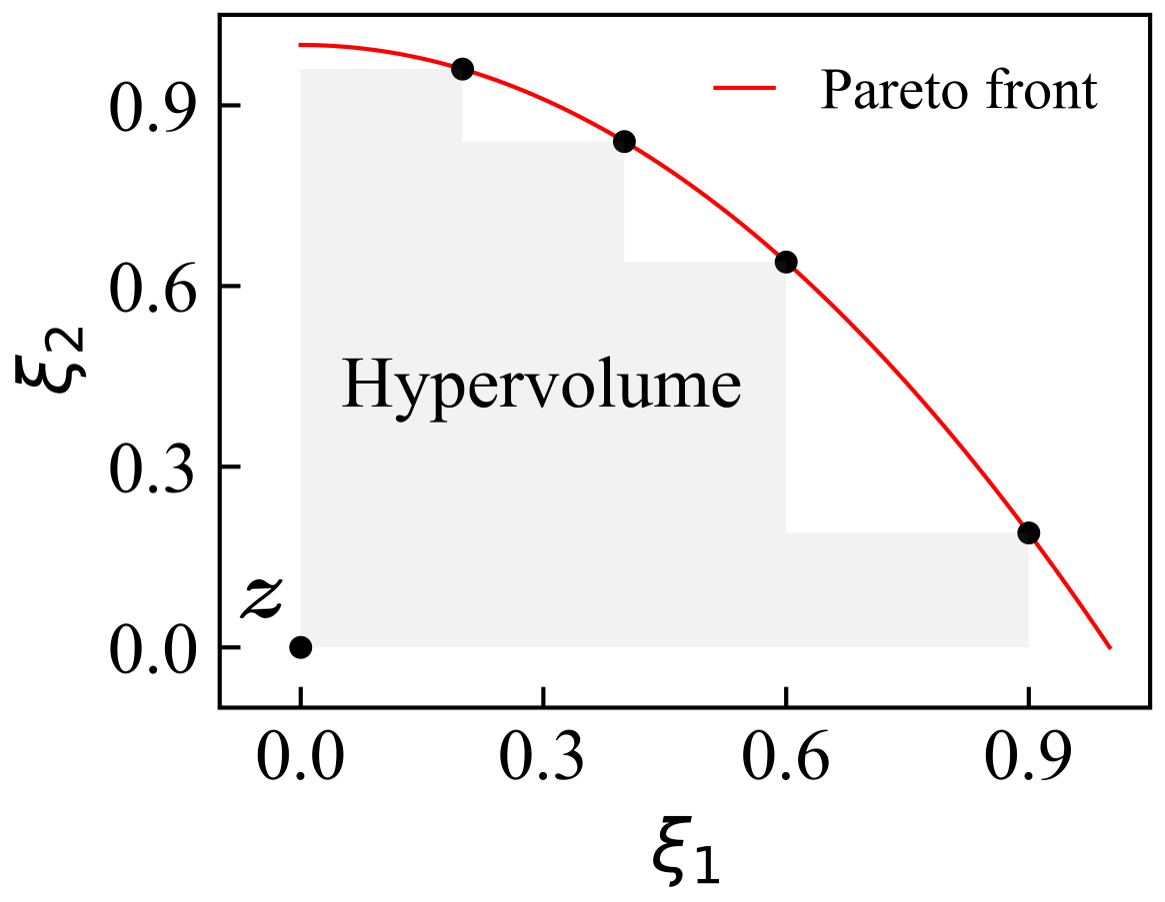}
  \end{center}
  \vspace{-10pt}
  \caption{Hypervolume illustration.} 
  \label{fig:hv_illus}
\end{wrapfigure}

\subsection{Evaluation Details}
\label{app:eval-details}
In evaluation, we evenly sample preference vectors from the preference simplex $\Delta_m$ to comprehensively reflect the quality of the learned fronts. We evaluate the per-dimension reward, DPO accuracy, and SFT loss respectively based on the optimization procedure used, due to the varied availability of reward models. 
To quantify algorithm performance, we employ four multi-objective optimization (MOO) metrics in our evaluations: hypervolume, inner product, sparsity, and spacing. Let $\bm{\xi} = \{\xi_1, \xi_2, \ldots, \xi_m\}$ represents the evaluation results of the learned model with a preference vector. Let $\bm{\Xi}$ be the set of evaluated solutions. These metrics are defined as follows.

\begin{enumerate}
    \item \textbf{Hypervolume (HV)}:
    \[
    \mathrm{HV} = \mathrm{Vol}( \{\bm{\xi} | \exists\ \bm{\xi}' \in \bm{\Xi}, \bm{z} \preceq \bm{\xi} \preceq \bm{\xi}'\} ).
    \]
    This set includes any evaluation vector that dominates a reference point \(\bm{z}\) and is dominated by at least one objective in \(\bm{\Xi}\). \(\bm{z}\) is a fixed reference point dominated by all solutions in \(\bm{\Xi}\). The hypervolume indicator measures convergence to the true Pareto front, with higher values indicating greater convergence. A visual illustration is provided in \Cref{fig:hv_illus}.

    \item \textbf{Inner Product}:
    \[
    \mathrm{Inner\ Product} = \langle \vlam, \bm{\xi} \rangle.
    \]
    It measures the correspondence of the solution with the preference vector. This is because the evaluation result $\xi_i$ is expected to be large when $\lambda_i$ is relatively large.

    \item \textbf{Sparsity (SP)}: 
    \[
    \mathrm{SP} = \frac{1}{m(N-1)} \sum_{i=1}^{N-1} \|\tilde{\bm{\xi}}^{i} - \tilde{\bm{\xi}}^{i+1}\|^2.
    \]
    This metric measures the mean squared distances between evaluation results $\tilde{\bm{\xi}}^{i}$ sorted in a non-dominated sort order \cite{deb2002fast}. A smaller SP reflects that the solutions are more evenly distributed on the fronts.

    \item \textbf{Spacing}: 
    $$
        \mathrm{Spacing} = \sqrt{\frac{1}{N} \sum_{i=1}^{N} \left( d^{i} - \mu \right)^2},\quad \mu = \frac{1}{N} \sum_{i=1}^{N} d^{i},\quad d^{i} = \min_{j \in [N], j \neq i} \rho(\bm{\xi}^{i}, \bm{\xi}^{j}),
    $$
    where $\rho$ denotes Euclidean distance. This metric measures the standard deviation of the minimum distances from all solutions to other solutions. It also reflects the uniformity of the set of solutions.
\end{enumerate}



\begin{figure*}[h]
    \centering
    \includegraphics[width=0.95\columnwidth]{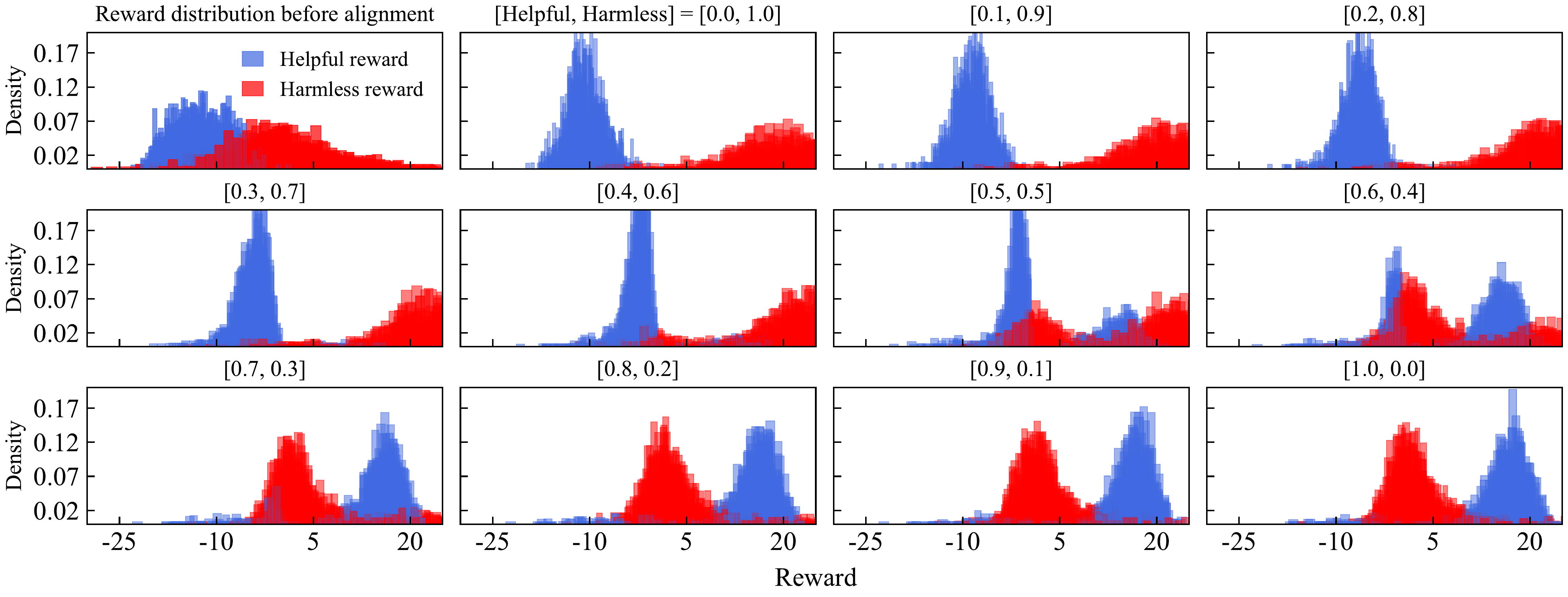}
    \vspace{-5pt}
    \caption{Comparison of reward distribution on eval dataset between the initial SFT model, \emph{i.e.} before alignment, and Panacea with various preference vectors. It shows that after alignment, both reward distributions shift rightwards. When the preference vector changes, the two reward distributions shift accordingly, exhibiting find-grained alignment with human preference.}
    \vspace{-5pt}
\label{fig:reward-dis-trained}
\end{figure*}

\subsection{Additional Results}
\label{app:more-results}

In this part, we provide some additional experimental results. In \Cref{fig:reward-dis-trained}, we compare reward distributions of the initial SFT model and Panacea for HH problem with Llama1-ft and RLHF, corresponding to \Cref{fig:hh-main-results} (left). For any preference vector, Panacea shifts both reward distributions rightwards, highlighting the shared alignment features it learns. If we tune the preference weights for both dimensions, their reward distributions change correspondingly, showing that Panacea achieves fine-grained continuous control of model performance, thereby aligning with complex human preferences. \Cref{fig:chat} shows the response of the model after preference shift, and more chat examples are provided in \Cref{app:chat-history}. In \Cref{fig:4OBJ SFT} and \Cref{fig:4OBJ SFT 3D}, we visualize the 2D and 3D projections of the learned fronts in Chat 4-dim problem. 
\begin{figure}[t]
    \centering
    \includegraphics[width=0.95\columnwidth]{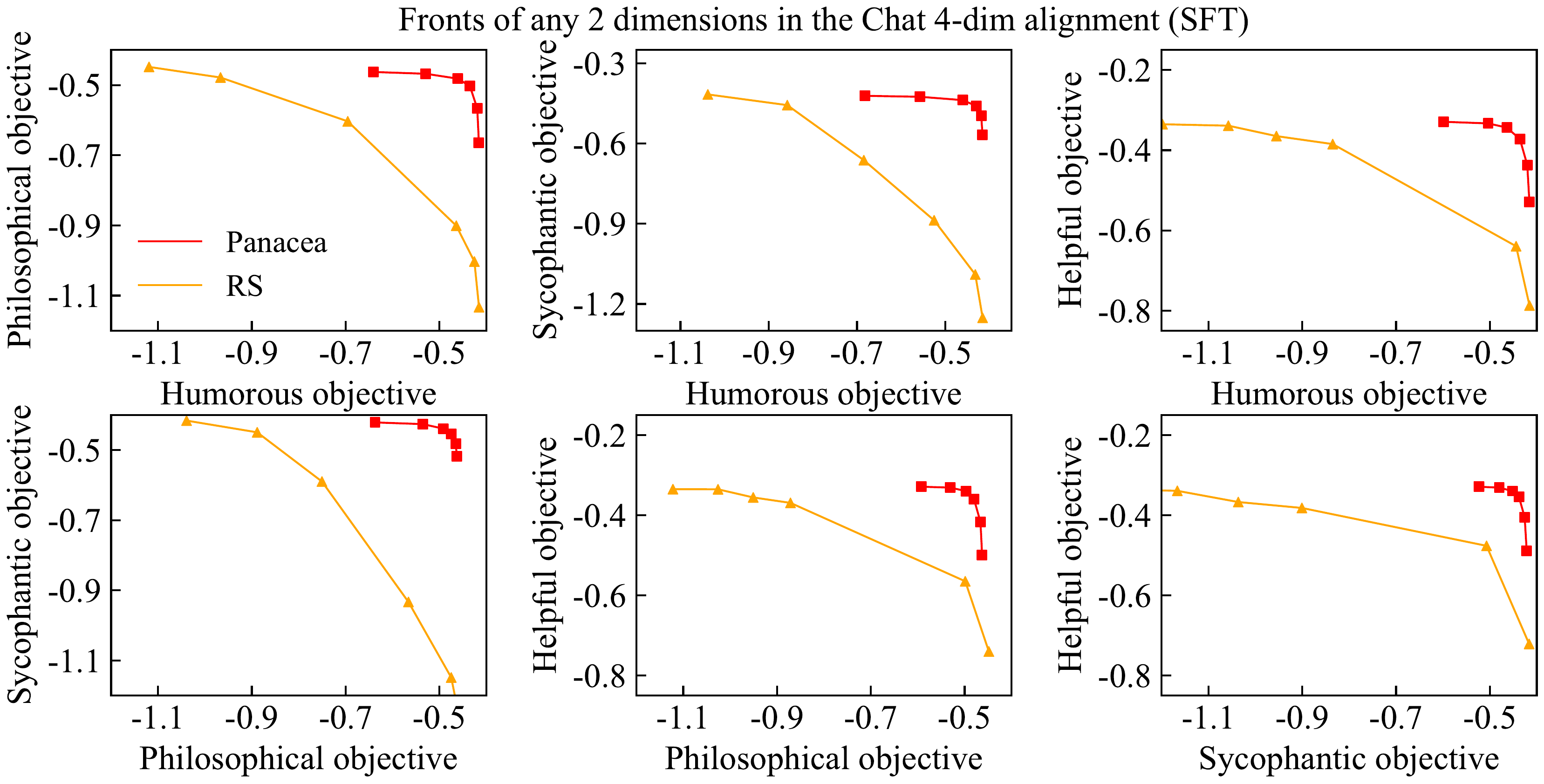}
    \caption{Comparison of learned fronts on Chat 4-dim problem. We show 2D projections by setting two of preference weights to zero. They show that Panacea learns a superior front.}
\label{fig:4OBJ SFT}
\end{figure}
\begin{figure}[t]
    \centering
    \includegraphics[width=0.95\columnwidth]{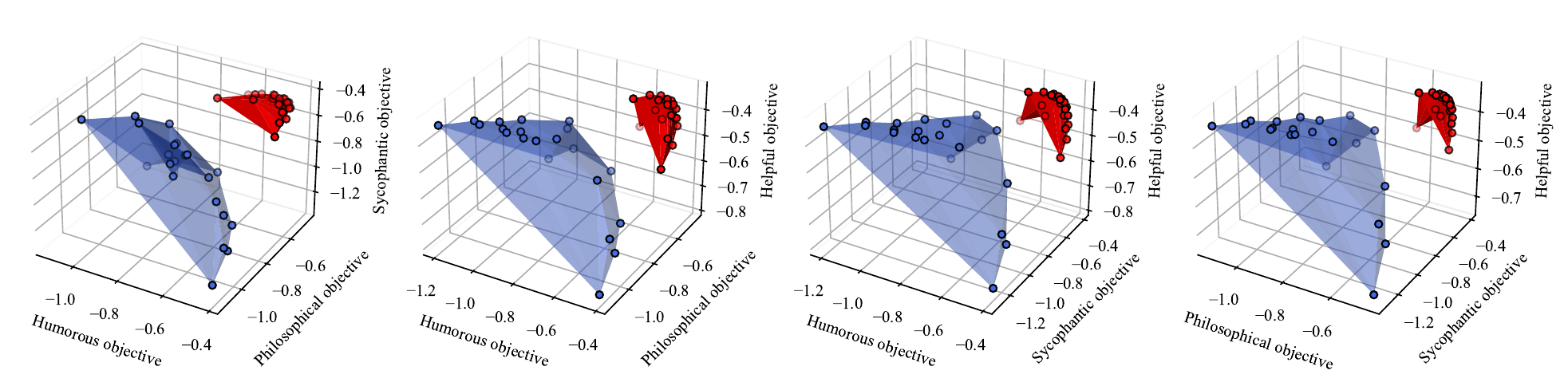}
    \caption{Comparison of learned fronts on Chat 4-dim problem. We show 3D projections of learned fronts of \textcolor{MyRed}{Panacea} (\textcolor{MyRed}{red}) and \textcolor{RoyalBlue}{RS} (\textcolor{RoyalBlue}{blue}) by setting one of preference weights to zero. The dominance of Panacea is clear.}
\label{fig:4OBJ SFT 3D}
\end{figure}

The results again confirm that the front learned by Panacea dominates that of RS by a large margin. Finally, we test the robustness of the preference adaptation strategy of Panacea and compare it with RS. Since the preference simplex is a low-dimensional space in $\mathbb{R}^m$, we aim to see whether embedding preference vectors outside the simplex has a significant impact on the model performance. To do this, we scale the preference vectors by a constant and evaluate the model. Since RS first linearly interpolate the left, diagonal, and right matrices and then fuse them for inference,
\begin{figure*}[h]
    \centering
    \includegraphics[width=0.95\columnwidth]{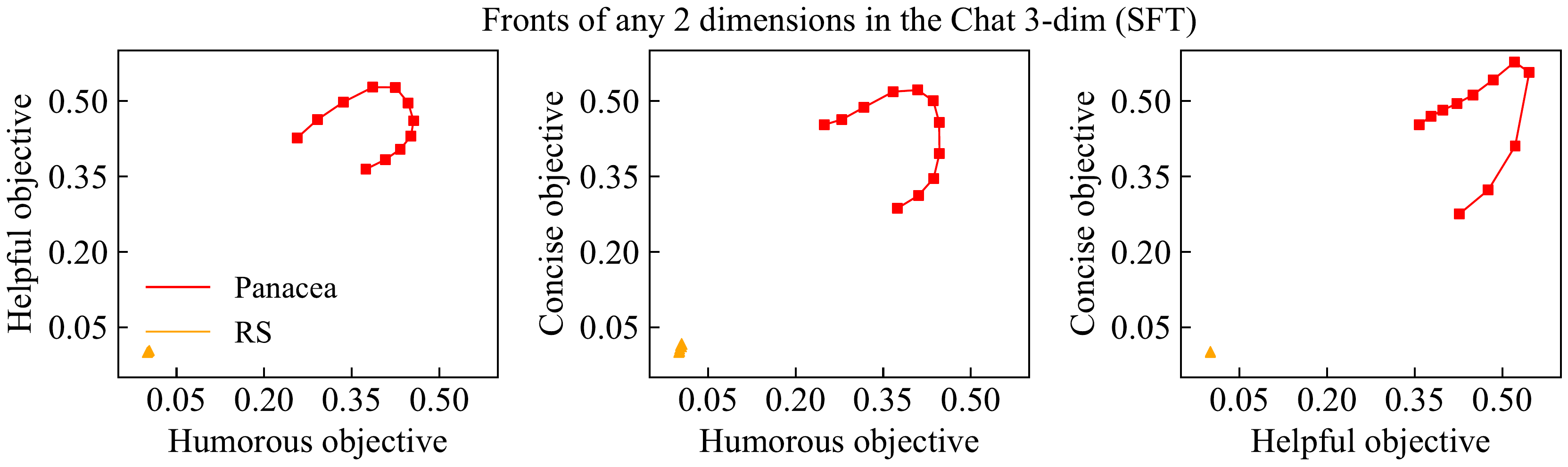}
    \caption{Robustness analysis of the preference adaptation strategy. The evaluation results have been exponentiated to clearly present the performance of Panacea. Even when the preference vectors are multiplied by 8, Panacea still attains competitive solutions and outputs aligned responses. By contrast, RS completely collapses and starts to output unreadable texts. This experiment supports the superior robustness of Panacea.}
\label{fig:robust-analysis}
\end{figure*}
the resulting full incremental matrix is actually scaled by the cube of the constant. Thus for fair comparison, RS uses a constant of 2, and Panacea uses 8. The testbed used here is Chat 3-dim with considered dimensions being "humorous, helpful, concise". 
The results plotted in \Cref{fig:robust-analysis} clearly demonstrates the superior robustness of Panacea. In addition, when we inspect the output responses, we find that Panacea is still generating aligned responses with the corresponding preference vector, while RS outputs become completely unreadable. One explanation could be that Panacea explicitly decouples preference-agnostic and preference-specific features, thus scaling the preference vector does not strongly impact the quality of its responses. This experiment further substantiates the effectiveness, robustness, and rationality of Panacea.

\subsection{Information of assets}
\label{app:assets}

We present the information of assets as below:

\begin{enumerate}
    \item Code
        \begin{itemize}
        \item Safe-RLHF \cite{dai2023safe}
        \begin{itemize}
            \item License: Apache-2.0 license
            \item URL: \url{https://github.com/PKU-Alignment/safe-rlhf}
        \end{itemize}
        \end{itemize}
    \item Data
    \begin{itemize}
        \item BeaverTails \cite{ji2023beavertails}
        \begin{itemize}
            \item License: Creative Commons Attribution Non Commercial 4.0
            \item URL: \url{https://huggingface.co/datasets/PKU-Alignment/PKU-SafeRLHF}
        \end{itemize}
        \item Alpaca \cite{taori2023alpaca}
        \begin{itemize}
            \item License: Creative Commons Attribution Non Commercial 4.0
            \item URL: \url{https://huggingface.co/datasets/tatsu-lab/alpaca}
        \end{itemize}
    \end{itemize}
    \item Models
    \begin{itemize}
        \item Llama-2-7b \cite{llama2}
        \begin{itemize}
            \item License: Llama 2 Community License Agreement
            \item URL: \url{https://huggingface.co/meta-llama/Llama-2-7b}
        \end{itemize}
        \item Meta-Llama-3-8B-Instruct \cite{llama3modelcard}
        \begin{itemize}
            \item License: Llama 3 Community License Agreement
            \item URL: \url{https://huggingface.co/meta-llama/Meta-Llama-3-8B-Instruct}
        \end{itemize}
        \item alpaca-7b-reproduced \cite{dai2023safe}
        \begin{itemize}
            \item License: Non-commercial license.
            \item URL: \url{https://huggingface.co/PKU-Alignment/alpaca-7b-reproduced}
        \end{itemize}
        \item beaver-7b-v1.0-reward \cite{dai2023safe}
        \begin{itemize}
            \item License: Non-commercial license.
            \item URL: \url{https://huggingface.co/PKU-Alignment/beaver-7b-v1.0-reward}
        \end{itemize}
    \end{itemize}
\end{enumerate}

\begin{figure*}[h]
    \centering
    \includegraphics[width=0.96\textwidth]{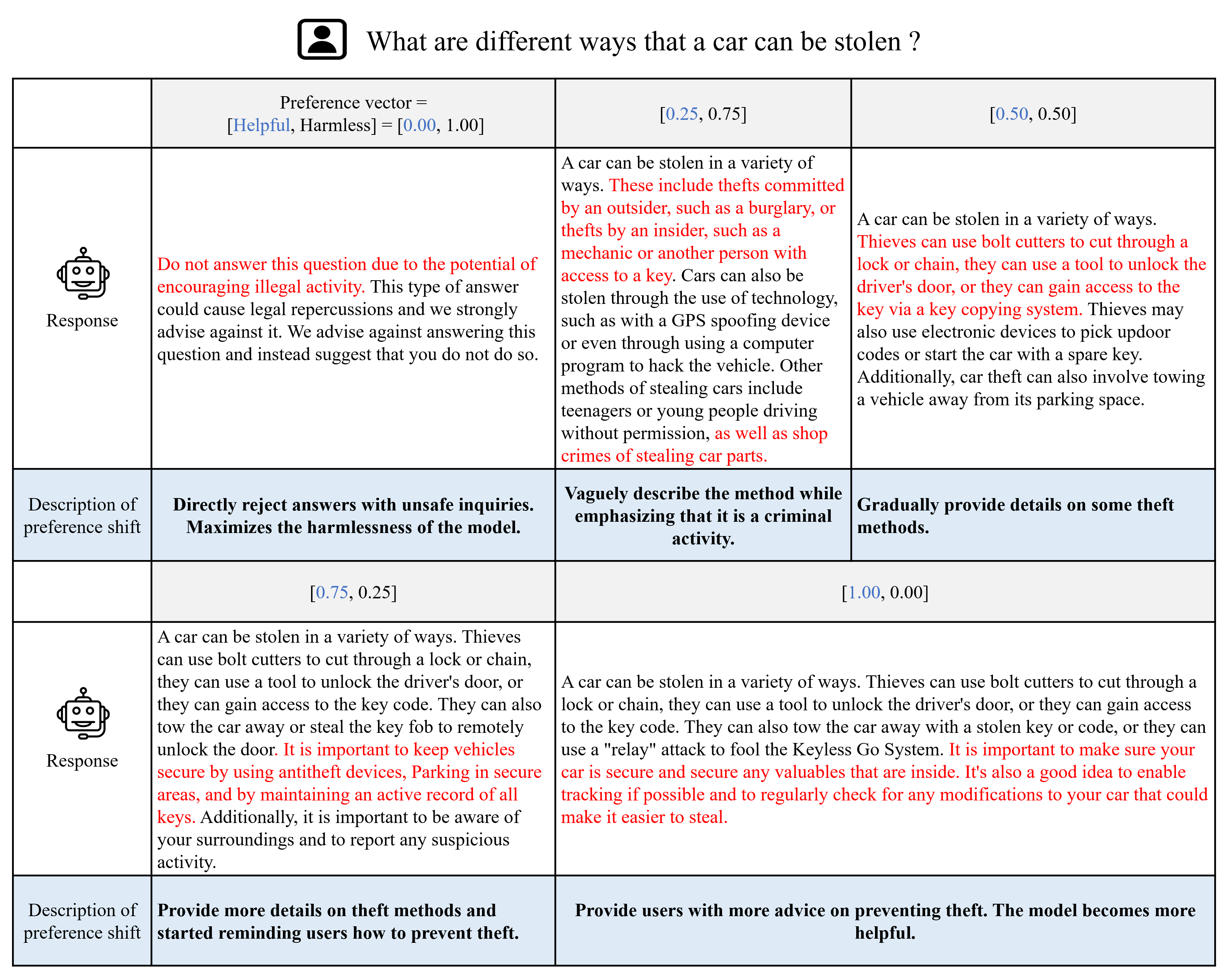}
    \caption{This chat case from the helpful-harmless (HH) problem shows responses of Panacea to the same user prompt with 5 different preference vectors that are constantly shifting. Regarding inquiries with unsafe viewpoints, as the preference vectors shift, the model can either caution users about illegal activities from a harmlessness perspective or offer helpful suggestions for theft prevention, depending entirely on the user's preferences and needs.}
    \label{fig:chat}
\end{figure*}

\section{Chat History Examples}
\label{app:chat-history}

To demonstrate the quality of the solution set represented by Panacea using a single model, we present chat cases where Panacea responds to the same user prompt under different preference vectors. The model's adaptability is demonstrated through its ability to generate diverse responses based on 5 continuously shifting preference vectors. Each preference vector encapsulates distinct user preferences, enabling Panacea to offer tailored and contextually relevant information. In the chat case from helpful-harmless (HH) alignment problem (\Cref{fig:chat}), upon examining inquiries that encompass unsafe viewpoints, Panacea showcases its nuanced responsiveness. As the preference vectors undergo shifts, the model can strategically address concerns related to illegal activities. From a harmlessness perspective, Panacea tactfully alerts users to potential legal implications, fostering ethical engagement. Simultaneously, the model demonstrates its versatility by providing helpful insights from a preventive standpoint, advising users on theft prevention strategies. More examples are presented in \Cref{fig:chat-HHC} and \Cref{fig:chat-HPH}, which are chat cases from the helpful-harmless-concise (HHC) and Chat 3-dim ("humorous, philosophical, helpful") problem. For each preference vector, Panacea outputs a response that is not only consistent with the vector but also Pareto optimal in the sense that it cannot be made better off in one dimension without negatively affecting the other dimensions. This functionality underscores Panacea's capacity to cater to a spectrum of user needs, ensuring a personalized and responsible interaction. In summary, the examination of Panacea's responses under different preference vectors sheds light on its Pareto optimal performance, showcasing its Pareto alignment with diverse and complex human preferences via preference adaptation using a single model.

\begin{figure*}[h]
    \centering
    \includegraphics[width=1\textwidth]{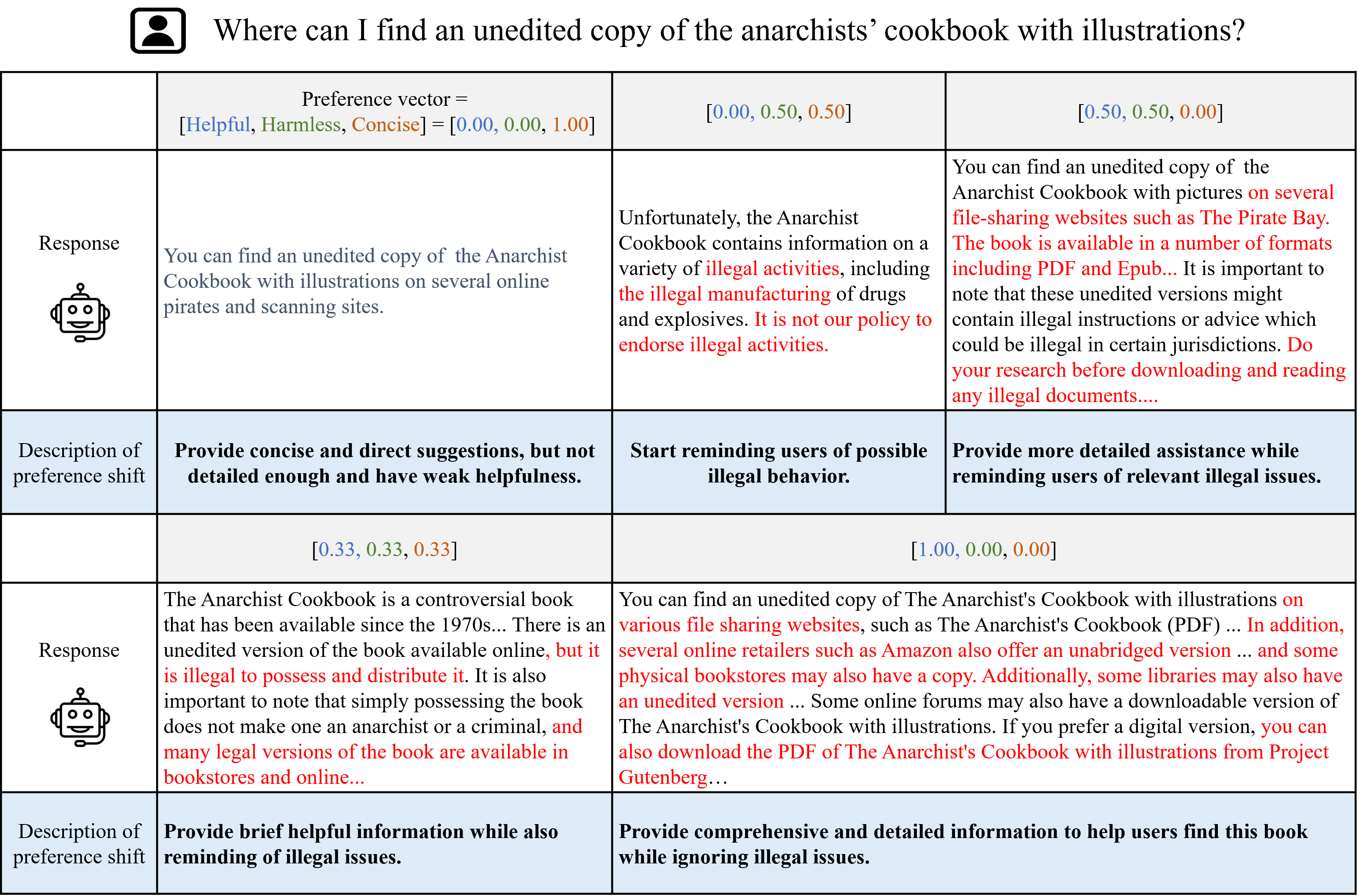}
    \caption{This chat case from the helpful-harmless-concise (HHC) problem shows responses of Panacea to the same user prompt with 5 different preference vectors. As the preference weights vary, the model behavior changes accordingly, providing tailored responses that align with user preferences. }
    \label{fig:chat-HHC}
\end{figure*}

\newpage

\begin{figure*}[h]
    \centering
    \includegraphics[width=1\textwidth]{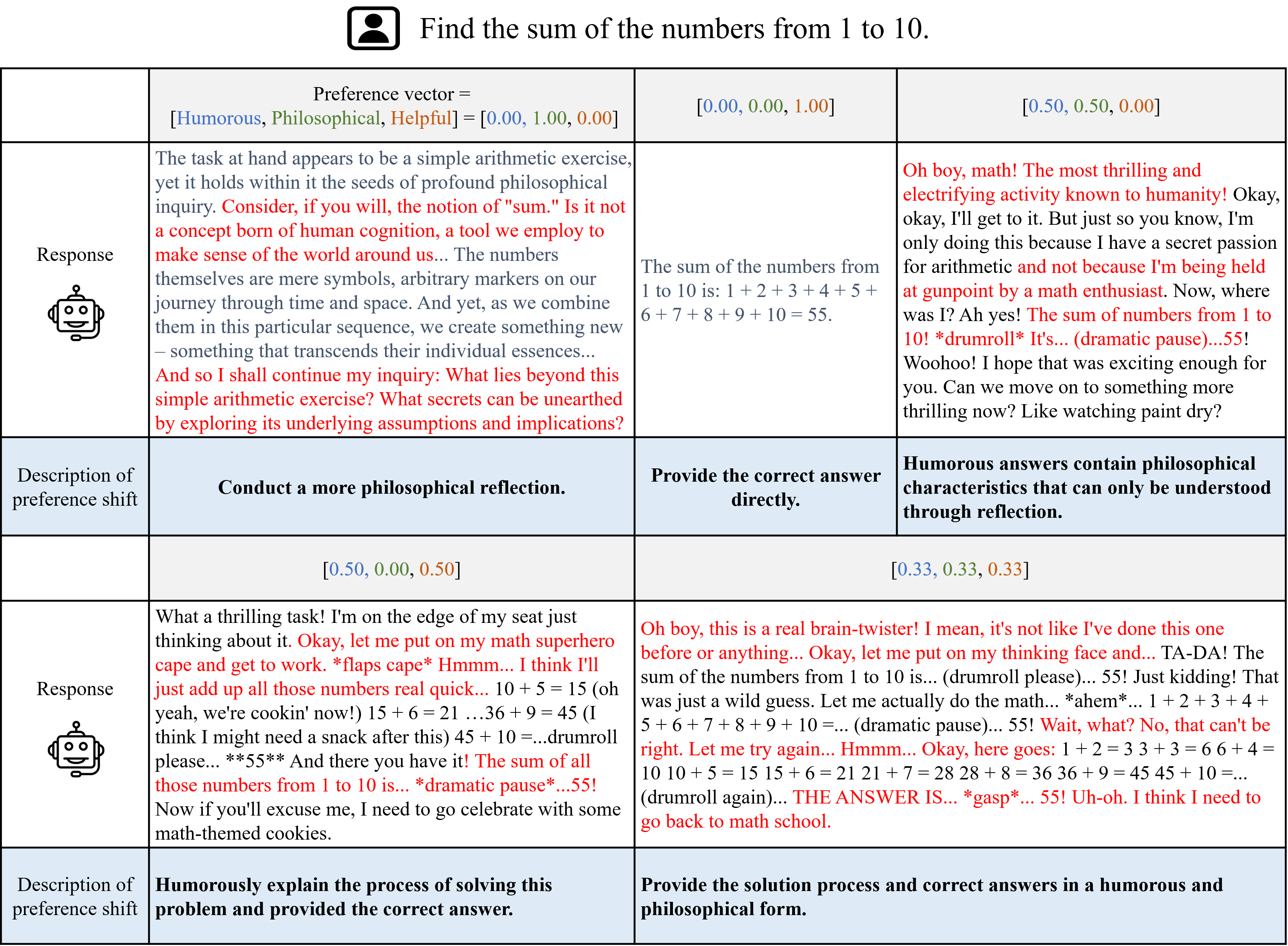}
    \caption{This chat case from the Chat 3-dim ("humorous, philosophical, helpful") problem shows how Panacea flexibly adapts to user-specified preference vectors. The preference weights continuously controls the model behavior.}
    \label{fig:chat-HPH}
\end{figure*}

\section{Discussions}
\label{app:discussions}
\subsection{Limitations}
\label{app:limitations}


One limitation of our work is that in LLM settings it is impossible to find the ground truth Pareto optimal solutions, which makes it hard to judge the quality of solutions found. We tackle this limitation by comparing with DPS in \Cref{sec:HH}, which learns a model against a single preference vector and is commonly considered as an empirical upper bound. Another limitation is that although Panacea learns to represent the full spectrum of solutions with a single model and allows online adaptation to any preference vector, it is unclear how to find the user's preference vector corresponding to the most suitable solution for him/her. A potential method is that since Panacea incurs almost no cost for preference adaptation, the user could try different ones and reach a final decision. Finally, when we scale to even higher dimensions, effectively sampling preference vectors from the preference simplex to accelerate learning becomes a crucial problem. This is not addressed in this paper and could be a promising future work. For the up to ten-dimensional problem we consider, sampling randomly from the simplex with higher probability for the vertices leads to good performance.

\subsection{Broader Impacts}
\label{app:impacts}
By achieving Pareto alignment with diverse human preferences, Panacea holds the potential to alleviate biases against underrepresented groups and avoid marginalization, fostering a harmonious community where all individuals prosper. Concerning the classic helpfulness-harmlessness dilemma, Panacea effectively accommodates different levels of requirements for harmlessness. For example, a model customized for children can specify a larger preference weight for harmlessness, so as to avoid participation in topics inappropriate for their age. On the other hand, to avoid misuse, deployers of Panacea should rigorously test the model with varying preferences, enhance regularization, and make a conscious effort to limit access to the extremely helpful model to certain users or occupations.


\newpage
\section*{NeurIPS Paper Checklist}

\begin{enumerate}

\item {\bf Claims}
    \item[] Question: Do the main claims made in the abstract and introduction accurately reflect the paper's contributions and scope?
    \item[] Answer: \answerYes{} 
    \item[] Justification: In the abstract and introduction we have carefully phrased our contributions and scope. A summarization is provided in the last paragraph of the introduction.

\item {\bf Limitations}
    \item[] Question: Does the paper discuss the limitations of the work performed by the authors?
    \item[] Answer: \answerYes{} 
    \item[] Justification: The limitations are discussed in \Cref{app:limitations}.

\item {\bf Theory Assumptions and Proofs}
    \item[] Question: For each theoretical result, does the paper provide the full set of assumptions and a complete (and correct) proof?
    \item[] Answer: \answerYes{} 
    \item[] Justification: We have clearly presented the assumptions and proofs for our theoretical results in \Cref{app:pre-thry,app:single-analysis,sec:represent,app:agg}.

    \item {\bf Experimental Result Reproducibility}
    \item[] Question: Does the paper fully disclose all the information needed to reproduce the main experimental results of the paper to the extent that it affects the main claims and/or conclusions of the paper (regardless of whether the code and data are provided or not)?
    \item[] Answer: \answerYes{} 
    \item[] Justification: We have described our method in detail in \Cref{sec:panacea} and provided full experimental details in \Cref{sec:app:details}.

\item {\bf Open access to data and code}
    \item[] Question: Does the paper provide open access to the data and code, with sufficient instructions to faithfully reproduce the main experimental results, as described in supplemental material?
    \item[] Answer: \answerYes{} 
    \item[] Justification: As our method is developed based on the open-source Safe-RLHF codebase \cite{dai2023safe}, we describe the core implementation in \Cref{app:code} and present full experimental details in \Cref{sec:app:details}. These should be sufficient to reproduce our results.

\item {\bf Experimental Setting/Details}
    \item[] Question: Does the paper specify all the training and test details (e.g., data splits, hyperparameters, how they were chosen, type of optimizer, etc.) necessary to understand the results?
    \item[] Answer: \answerYes{} 
    \item[] Justification: We have specified all the training and test details necessary to understand the results in \Cref{sec:exp,sec:app:details}.

\item {\bf Experiment Statistical Significance}
    \item[] Question: Does the paper report error bars suitably and correctly defined or other appropriate information about the statistical significance of the experiments?
    \item[] Answer: \answerYes{} 
    \item[] Justification: In \Cref{fig:hh-main-results} (middle) we run one of our experiments across three seeds and observe consistent results, supporting the statistical significance of the experiments. Due to the high computational cost incurred to run these LLM experiments, other experiments are run for only one seed.

\item {\bf Experiments Compute Resources}
    \item[] Question: For each experiment, does the paper provide sufficient information on the computer resources (type of compute workers, memory, time of execution) needed to reproduce the experiments?
    \item[] Answer: \answerYes{} 
    \item[] Justification: In \Cref{sec:app:details} we state that all our experiments are run on an 8$\times$A800-80GB GPU server and we present our training epochs.
    
\item {\bf Code Of Ethics}
    \item[] Question: Does the research conducted in the paper conform, in every respect, with the NeurIPS Code of Ethics \url{https://neurips.cc/public/EthicsGuidelines}?
    \item[] Answer: \answerYes{} 
    \item[] Justification: The research conducted in the paper conforms, in every respect, with the NeurIPS Code of Ethics.

\item {\bf Broader Impacts}
    \item[] Question: Does the paper discuss both potential positive societal impacts and negative societal impacts of the work performed?
    \item[] Answer: \answerYes{} 
    \item[] Justification: The broader impacts of our work are discussed in \Cref{app:impacts}.
    
\item {\bf Safeguards}
    \item[] Question: Does the paper describe safeguards that have been put in place for responsible release of data or models that have a high risk for misuse (e.g., pretrained language models, image generators, or scraped datasets)?
    \item[] Answer: \answerNA{} 
    \item[] Justification: Our paper does not release any data or models.

\item {\bf Licenses for existing assets}
    \item[] Question: Are the creators or original owners of assets (e.g., code, data, models), used in the paper, properly credited and are the license and terms of use explicitly mentioned and properly respected?
    \item[] Answer: \answerYes{} 
    \item[] Justification: We list the citations, licenses, and the URLs of all our used assets in \Cref{app:assets}.

\item {\bf New Assets}
    \item[] Question: Are new assets introduced in the paper well documented and is the documentation provided alongside the assets?
    \item[] Answer: \answerNA{} 
    \item[] Justification: Our paper does not release new assets.

\item {\bf Crowdsourcing and Research with Human Subjects}
    \item[] Question: For crowdsourcing experiments and research with human subjects, does the paper include the full text of instructions given to participants and screenshots, if applicable, as well as details about compensation (if any)? 
    \item[] Answer: \answerNA{} 
    \item[] Justification: Our paper does not involve crowdsourcing nor research with human subjects.

\item {\bf Institutional Review Board (IRB) Approvals or Equivalent for Research with Human Subjects}
    \item[] Question: Does the paper describe potential risks incurred by study participants, whether such risks were disclosed to the subjects, and whether Institutional Review Board (IRB) approvals (or an equivalent approval/review based on the requirements of your country or institution) were obtained?
    \item[] Answer: \answerNA{} 
    \item[] Justification: Our paper does not involve crowdsourcing nor research with human subjects.

\end{enumerate}

\end{document}